\documentclass[12pt,draftclsnofoot,onecolumn]{IEEEtran}
\usepackage{graphicx}
\usepackage{subfigure}
\usepackage{amsmath,amsthm,amsfonts}
\usepackage{cite}
\usepackage{bm}
\usepackage{url}
\usepackage{array}
\usepackage{color}
\usepackage{multirow}
\usepackage{booktabs}
\usepackage[ruled]{algorithm2e}
\usepackage{setspace}

\theoremstyle{plain}
\newtheorem{thm}{Theorem}
\newtheorem{lem}[thm]{Lemma}

\newtheorem{prop}[thm]{Proposition}

\newtheorem*{con}{Conjecture}

\allowdisplaybreaks[4]

\makeatletter
\def\old@comma{,}
\catcode`\,=13
\def,{%
  \ifmmode%
    \old@comma\discretionary{}{}{}%
  \else%
    \old@comma%
  \fi%
}
\makeatother

\begin{document}
\title{AlphaSeq: Sequence Discovery with Deep Reinforcement Learning}

\author{
Yulin~Shao,~\IEEEmembership{Student Member,~IEEE,}
Soung~Chang~Liew,~\IEEEmembership{Fellow,~IEEE,}
and~Taotao~Wang,~\IEEEmembership{Member,~IEEE}
\thanks{Y. Shao, S. C. Liew and T. Wang are with the Department of Information Engineering, The Chinese University of Hong Kong, Shatin, New Territories, Hong Kong (e-mail: \{sy016, soung, ttwang\}@ie.cuhk.edu.hk).}
}
\maketitle

\begin{abstract}
Sequences play an important role in many applications and systems.
Discovering sequences with desired properties has long been an interesting intellectual pursuit.
This paper puts forth a new paradigm, AlphaSeq, to discover desired sequences algorithmically using deep reinforcement learning (DRL) techniques. AlphaSeq treats the sequence discovery problem as an episodic symbol-filling game, in which a player fills symbols in the vacant positions of a sequence set sequentially during an episode of the game.
Each episode ends with a completely-filled sequence set, upon which a reward is given based on the desirability of the sequence set.
AlphaSeq models the game as a Markov Decision Process (MDP), and adapts the DRL framework of AlphaGo to solve the MDP.
Sequences discovered improve progressively as AlphaSeq, starting as a novice, learns to become an expert game player through many episodes of game playing.
Compared with traditional sequence construction by mathematical tools, AlphaSeq is particularly suitable for problems with complex objectives intractable to mathematical analysis.
We demonstrate the searching capabilities of AlphaSeq in two applications:
1) AlphaSeq successfully rediscovers a set of ideal complementary codes that can zero-force all potential interferences in multi-carrier CDMA systems.
2) AlphaSeq discovers new sequences that triple the signal-to-interference ratio -- benchmarked against the well-known Legendre sequence -- of a mismatched filter estimator in pulse compression radar systems.
\end{abstract}


\section{Introduction}
A sequence is a list of elements arranged in a certain order.
Prime numbers arranged in ascending order, for example, is a sequence \cite{OEIS}.
The arrangements of nucleic acids in DNA polynucleotide chains are also sequences \cite{DNASeq}.

Discovering sequences with desired properties is an intellectual pursuit with important applications \cite{OEIS}.
In particular, sequences are critical components in many information systems.
For example, cellular code division multiple access (CDMA) systems make use of spread spectrum sequences to distinguish signals from different users \cite{CDMA1992};
pulse compression radar systems make use of probe pulses modulated by phase-coded sequences \cite{RadarBook} to enable high-resolution detection of objects at a large distance.

Sequences in information systems are commonly designed by algebraists and information theorists using mathematical tools such as finite filed theory, algebraic number theory, and character theory.
However, the design criterion for a good sequence may be complex and cannot be put into a clean mathematical expression for solution by the available mathematical tools.
Faced with this problem, sequence designers may do two things:
1) Overlook the practical criterion and simplify the requirements to make the problems analytically tractable. In so doing, a disconnect between reality and theory may be created.
2) Introduce additional but artificial constraints absent in the original practical problem. In this case, the analytical solution is only valid for a subset of sequences of interest.
For example, the protocol sequences in \cite{CRT2018} are constructed by means of the Chinese Remainder Theorem (CRT) \cite{NumberTheoryBook}; hence, the number of supported users is restricted to a prime number.

Yet a third approach is to find the desired sequences algorithmically.
This approach rids us of the confines imposed by analytical mathematical tools.
On the other hand, the issue becomes whether good sequences can be found within a reasonable time by algorithms.
Certainly, to the extent that desired sequences can be found by a random search algorithm within a reasonable time, then the problem is solved.
Most desired sequences, however, cannot be found so easily and algorithms with complexity polynomial in the length of the sequences are not available.

\begin{figure}[t]
  \centering
  \includegraphics[width=0.6\columnwidth]{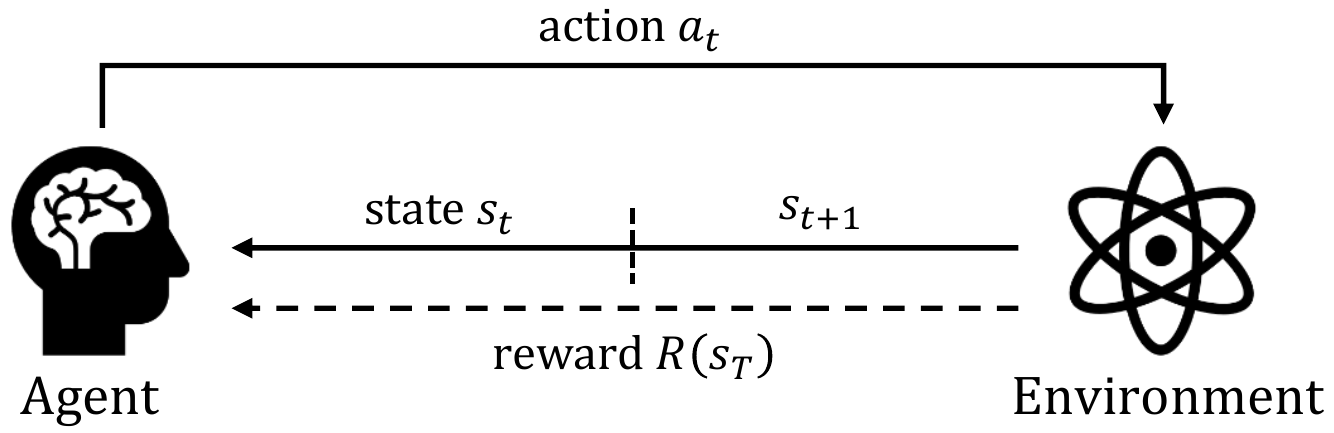}\\
  \caption{In episodic reinforcement learning, the agent-environment interactions are broken into sessions called episodes. Each episode starts anew from an initial state $s_0$. The agent takes actions in successive discrete time steps $t=0,1,2,...,T-1$,  resulting in the state of the environment traversing through states $s_0,s_1,s_2,...,$ until a terminal state $s_T$ is reached, whereupon a reward $R(s_T)$ is given. The next episode begins independently of how the previous episode ended \cite{RLbook}.}
\label{Intro_RL}
\end{figure}

Reinforcement Learning (RL) is an important branch of machine learning \cite{RLbook} known for its ability to derive solutions for Markov Decision Processes (MDPs) \cite{MDPBook} through a learning process.
A salient feature of RL is ``learning from interactions''.
Fig. \ref{Intro_RL} illustrates the framework of episodic RL\footnote{In a general RL setup, a reward $R_{t+1}$ can be given at each time step to measure the quality of action $a_t$ in state $s_t$. The agent's objective at state $s_t$ is then to maximize the accumulated future reward $\sum_{i} \gamma^{i-1}R_{t+i}$, where $\gamma$ is a discount factor. The RL framework in Fig. \ref{Intro_RL}, by contrast, restricts the reward to only terminal state $s_T$, and sets $\gamma=1$. In other words, the goal at each state is now to maximize the end-of-episode reward $s_T$. As a matter of fact, for our problem, we do not have an exact measure of the quality of the sequence until all the elements of the are fixed, hence the use of the delayed reward set-up.}.
In the framework, an agent interacts with an environment in a sequence of discrete time steps $t=0,1,...,T-1$.
At time step $t$, the agent observes that the environment is in state $s_t$. Based on the observation of $s_t$, the agent then takes an action $a_t$, which results in the environment moving to state $s_{t+1}$ in time step $t+1$. The environment will feedback a reward $R(s_T)$ to the agent at the terminal state $s_T$, i.e., the end of one episode.

The mapping from $s_t$ to $a_t$ is referred to as a policy function. The aim of the policy is generally to maximize the expected reward received at the end of the episode.
This policy function could be deterministic, in which case a specific action $a_t$ is always taken upon a given state $s_t$.
The policy could also be probabilistic, in which case the action taken upon a given state is described by a conditional probability $P(a_t\mid s_t)$.
The objective of the agent is to learn an expected-reward maximizing policy after going through multiple episodes.\footnote{RL shares the same mathematical principle as that of dynamic programming (DP). To learn the optimal policy, RL algorithms typically contain two interacting processes: policy evaluation and policy improvement. We refer the reader to the exposition in \cite{RLbook}, in which Section 4.1 explains how policy evaluation predicts the state-value function for an arbitrary policy, and Section 4.2 explains how policy improvement improves the policy with respect to the current state-value function. Overall, these two processes interact with each other as a generalized policy iteration (Section 4.6), enabling the convergence to the optimal value function and an optimal policy.}.
The agent may begin with bad policies early on, but as it gathers experiences from successive episodes, the policy gets better and better.

The latest trend in RL research is to integrate the recent advances of deep learning \cite{NatureDL} into the RL framework \cite{Atari2015}, \cite{Go2016}, \cite{DRLsurvey}.
RL that makes use of deep neural networks (DNNs) to approximate the optimal policy function -- directly or indirectly -- is referred to as deep reinforcement learning (DRL).
DRL allows RL algorithms to be applied when the number of possible state-action pairs is enormous and that traditional function approximators cannot approximate the policy function accurately.
The recent success of DRL in game playing, natural language processing, and autonomous vehicle steering (see the excellent survey in \cite{DRLsurvey}) have demonstrated its power in solving complex problems that thwart conventional approaches.

This paper puts forth a DRL-based paradigm, referred to as AlphaSeq, to discover a set of sequences with desired properties algorithmically. The essence of AlphaSeq is as follows:
\begin{itemize}
\item AlphaSeq treats sequence-set discovery -- a sequence set consists of one or more sequences -- as an episodic symbol-filling game. In each episode of the game, AlphaSeq fills symbols into vacant sequence positions in a consecutive manner until the sequence set is completely filled, whereupon a reward with value between $-1$ and $1$ is returned. The reward is a nonlinear function of a metric that quantifies the desirability of the sequence set. AlphaSeq aims to maximize the reward. It learns to do by playing many episodes of the game, improving itself along the way.
\item AlphaSeq treats each intermediate state with some sequence positions filled and others vacant as an image. Each position is a pixel of the image. Given an input image (state), AlphaSeq makes use of a DNN to approximate the optimal policy that maximizes the reward. AlphaSeq uses a DRL framework similar to that of AlphaGo \cite{GoZero2017}, in which DNN-guided MCTS (Monte-Carlo Tree Search \cite{MCTS}) is used to select each move in the game. As in AlphaGo, there is an iterative self-learning process in AlphaSeq in that the experiences from the DNN-guided MCTS game playing are used to train the DNN; and the trained DNN in turn improves future game playing by the DNN-guided MCTS.
\item We introduce two techniques in AlphaSeq that are absent in AlphaGo for our applications to search sequences. The first technique is to allow AlphaSeq to make $\ell$ moves at a time (i.e., filling $\ell$ sequence positions at a time). Obviously, this technique is not applicable to the game of Go, hence AlphaGo. The choice of $\ell$ is a complexity tradeoff between the MCTS and the DNN. The second technique, dubbed ``segmented induction'', is to change the reward function progressively to guide AlphaSeq toward good sequences in its learning process. In essence, we set a low target for AlphaSeq initially so that many sequence sets can have rewards close to $1$, with few having rewards close to $-1$. As AlphaSeq plays more and more episodes of the game, we progressively raise the target so that fewer and fewer sequence sets have rewards close to $1$, with more having rewards close to $-1$. In other words, the game becomes more and more demanding as AlphaSeq, starting as a novice, learns to become an expert player.
\end{itemize}

We demonstrate the capability of AlphaSeq to discover two types of sequences:
\begin{enumerate}
\item We use AlphaSeq to rediscover a set of complementary codes for multi-carrier CDMA systems. In this application, AlphaSeq aims to discover a sequence set for which potential interferences in the multi-carrier CDMA system can be cancelled by simple signal processing. This particular problem already has analytical solutions. Our goal here is to test if AlphaSeq can rediscover these analytical solutions algorithmically rather than analytically.

\item We use AlphaSeq to discover new phase-coded sequences superior to the known sequences for pulse compression radar systems. Specifically, our goal is to find phase-coded sequences commensurate with the mismatched filter (MMF) estimator so that the estimator can yield output with high signal-to-interference ratio (SIR). The optimal sequences for MMF are not known and there is currently no known sequence that are provably optimal when the sequence is large. Benchmarked against the Legendre sequence \cite{Legendre}, the sequence discovered by AlphaSeq triples the SIR, achieving $5.23$ dB mean square error (MSE) gains for the estimation of radar cross sections in pulse compression radar systems.
\end{enumerate}
The remainder of this paper is organized as follows.
Section \ref{sec:methodology} formulates the sequence discovery problem and outlines the DRL framework of AlphaSeq.
Section \ref{sec:MCCDMA} and \ref{sec:PCR} present the applications of AlphaSeq in multi-carrier CDMA systems and pulse compression radar systems, respectively.
Section \ref{sec:Conclusion} concludes this paper.
Throughout the paper, lowercase bold letters denote vectors and uppercase bold letters denote matrices.

\section{Methodology}\label{sec:methodology}

\subsection{Problem Formulation}
We consider the problem of discovering a sequence set $\mathcal{C}$, the desirability of which is quantified by a metric $\mathcal{M(C)}$.
Set $\mathcal{C}$ consists of $K$ different sequences of the same length $N$, i.e., $\{\bm{c}_k:~k=0,1,...,K-1\}$, where the $k$-th sequence is given by $\bm{c}_k=(c_k[0],c_k[1],...,c_k[N-1])$.
Each symbol of the sequences in  $\mathcal{C}$ (i.e., $c_k[n]$) is drawn from a discrete set $\mathcal{A}$.
Without loss of generality, this paper focuses on binary sequences.
That is, $\mathcal{A}$ is two-valued, and we can simply denote these two values by $1$ and $-1$.
The metric function $\mathcal{M(C)}$ varies with application scenarios.
It is generally a function of all $K$ sequences in $\mathcal{C}$.
The optimal metric value $\mathcal{M}^*$ (i.e., the desired metric value) is achieved when $\mathcal{C}=\mathcal{C}^*$.
Our objective is to find an optimal sequence set $\mathcal{C}^*$ that yields $\mathcal{M}^*$.
For binary sequences, the complexity of exhaustive search for $\mathcal{C}^*$ is $\mathcal{O}(2^{NK})$, which is prohibitive for large $N$ and $K$.

This sequence discovery problem can be transformed into a MDP.
Specifically, we treat sequence-set discovery as a symbol-filling game.
One play of the game is one episode, and each episode contains a series of time steps. In each episode, the player (agent) starts from an all-zero state (i.e., all the symbols in the set are $0$), and takes one action per time step based on its current action policy.
In each time step, $\ell$ symbols in the sequence set are assigned with the value of $1$ or $-1$, replacing the original $0$ value.
We emphasize that the player can only determine the values of the $\ell$ symbols, but not their positions.
The $\ell$ positions are predetermined: a simple rule is to place symbols sequence by sequence (specifically, we first place symbols in one sequence. When this sequence is completed-filled, we turn to fill the next sequence, and so on so forth. This rule will be used throughout the paper unless specified otherwise).
An episode ends at a terminal state after $\lceil NK/\ell \rceil$ time steps, whereupon a complete set $\mathcal{C}$ is obtained.
In the terminal state, we measure the goodness of $\mathcal{C}$ by $\mathcal{M(C)}$, and return a reward $\mathcal{R(C)}$ for this episode to the player, where $\mathcal{R(C)}$ is in general a nonlinear function of $\mathcal{M(C)}$.
It is the player's objective to learn a policy that makes sequential decisions to maximizes the reward, as more and more games are played.

\subsection{Methodology}
Given the MDP, a tree can be constructed by all possible states in the game.
In particular, the root vertex is the all-zero state, and each vertex of the tree corresponds to a possible state, i.e., a partially-filled sequence-set pattern (completely-filled at a terminal state).
The depth of the tree equals the number of time steps in an episode (i.e., $\lceil NK/\ell \rceil$), and each vertex has exactly $2^\ell$ branches. In each episode, the player will start from the root vertex, and make sequential decisions along the tree based on its current policy until reaching a leaf vertex, whereupon a reward will be obtained. Given any vertex $v_i$ and an action, the next vertex $v_{i+1}$ is conditionally independent of all previous vertices and actions, i.e., the transitions on the tree satisfy the Markov property.

The objective of the player is then to reach a leaf vertex with the maximum reward. Toward this objective, the player performs the following:
\begin{enumerate}
\item Distinguishing good states from bad states -{}-
A reward is given to the player only upon its reaching a terminal stage.
While traversing the intermediate stage, the player must distinguish good intermediate states from bad intermediate states so that it can navigate toward a good terminal stage. In particular, the player must learn to approximate the expected end rewards of intermediate states: this is in fact a process of value function approximation (In RL, the value of a state refers to the expected reward of being in that state, and a value function is a mapping from states to values. For terminal states, the value function is exactly the reward function.). Moreover, we can imagine each state to be an image with each symbol being a pixel, and make use of a DNN to approximate the expected rewards of the ``images''.
\item Improving action policy based on cognition of subsequent states.
Starting as a tabula rasa, the player's initial policy in earlier episodes is rather random.
To gradually improve the action policy, the player can leverage the instrument of MCTS.
MCTS is a simulated look-ahead tree search.
At a vertex, MCTS can estimate the prospects of subsequent vertices by simulating multiple actions along the tree.
The information collected during the simulations can then be used to decide the real action to be taken at this vertex\footnote{The main concept in MCTS is tree policy. It determines how we sample the tree and select nodes. For a general overview on the core algorithms and variations, we refer the reader to the excellent survey \cite{MCTS} (in particular, the most popular algorithm in the MCTS family, the Upper Confidence Bound for Trees (UCT), is introduced in Section 3.3 of \cite{MCTS}). Ref. \cite{MCTSproof} provides a more rigorous proof of the optimality of UCT. The authors showed that the probability that the UCT selects the optimal action converges to $1$ at a polynomial rate.}.
\end{enumerate}

A successful combination of DNN and MCTS has been demonstrated in AlphaGo \cite{Go2016}, \cite{GoZero2017}, \cite{AlphaZero2017}, where the authors use DNN to assess the vertices during the MCTS simulation, as opposed to using random rollouts in standard MCTS\footnote{More details on the standard MCTS can be found in \cite{MCTS}. Throughout the paper, when we refer to MCTS, we mean the DNN-guided MCTS rather than the standard MCTS.}.
In this paper, we adapt the DRL framework in AlphaGo\footnote{AlphaGo itself is evolving, the DRL framework in this paper is based on AlphaGo Zero \cite{GoZero2017} and AlphaZero \cite{AlphaZero2017}.} to solve the sequence set discovery problem associated with the underlying MDP. In deference to AlphaGo, we refer to this sequence discovering framework as ``AlphaSeq''.

\begin{figure}[t]
  \centering
  \includegraphics[width=0.6\columnwidth]{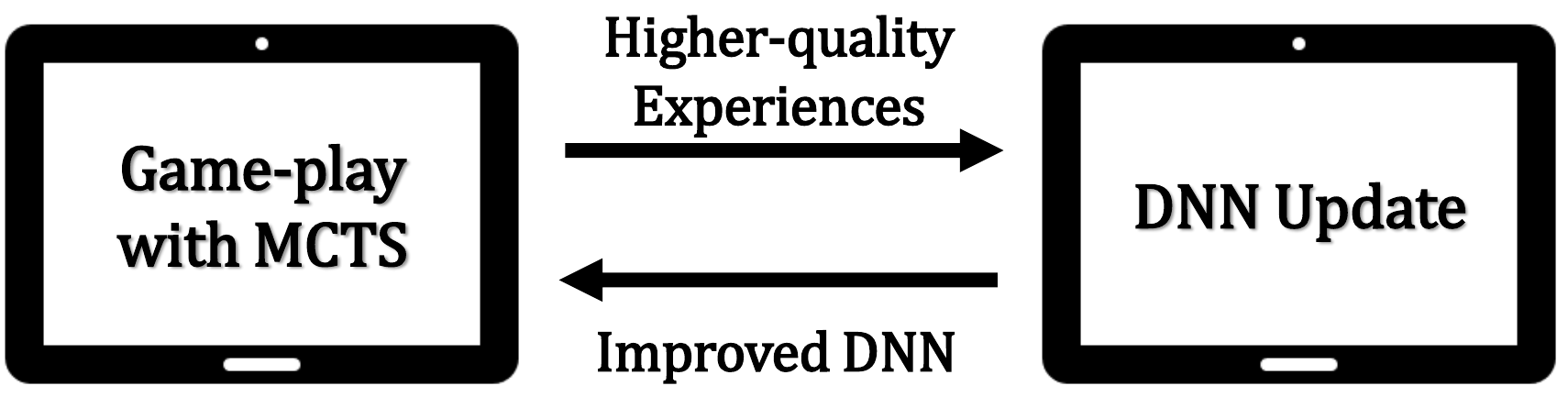}\\
  \caption{The iterative algorithmic framework of AlphaGo/AlphaSeq. Improved DNN promotes the MCTS so that ``game-play'' generates experiences with higher quality; higher quality experiences can further enhance the DNN.}
\label{SecII_framework}
\end{figure}
The overall algorithmic framework of AlphaGo/AlphaSeq can be outlined as an iterative ``game-play with MCTS'' and ``DNN-update'' process, as shown in Fig. \ref{SecII_framework}.
On the one hand, ``game-play with MCTS'' provides experiences to train the DNN so that the DNN can improves its assessments of the goodness of the states in the game.
On the other hand, better evaluation on the states by the DNN allows the MCTS to make better decisions, which in turn provide higher quality experiences to train the DNN.
Through an iterative process, the MCTS and the DNN mutually enhances each other in a progressive manner over an underlying reinforcement learning process.

In what follows, we dissect these two components and describe the relationship between them with more details. Differences between AlphaSeq and AlphaGo are presented at the end of this section. Further implementation details can be found in Appendix \ref{sec:AppA}.

\textbf{Input and output of DNN} -{}- The DNN is designed to estimate the value function and policy function of an intermediate state\footnote{The DNN will only evaluate intermediate states, but not the terminal states. For terminal states, the value function is known to the player (i.e., the reward function), and there is no policy.}.
The value function is the estimated expected terminal reward given the intermediate state.
Specifically, the output of DNN can be expressed as $(\bm{P}, \mathcal{R}')=\psi_{\theta}(s_i)$: each time we feed an intermediate state $s_i$ into the DNN $\psi$ with coefficients $\theta$, it will output a reward estimation $\mathcal{R}'$ (value function estimation) and a probabilistic move-selection policy $\bm{P}$ (policy function estimation, policy $\bm{P}$ is a distribution over all possible next moves given the current state $s_i$).

\begin{figure}[t]
  \centering
  \includegraphics[width=0.6\columnwidth]{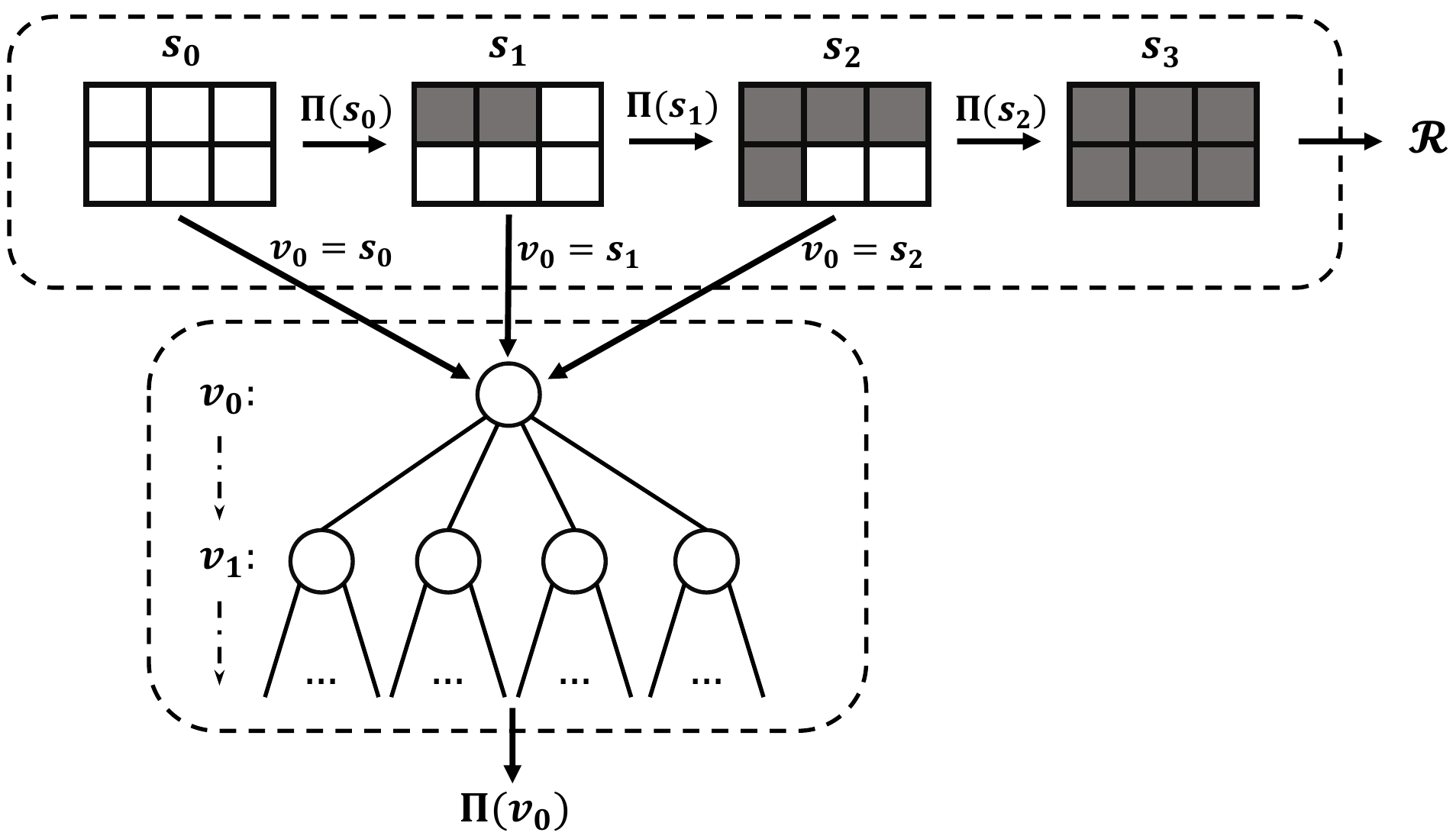}\\
  \caption{An episode of game, where $K=2$, $N=3$, and $\ell=2$. The $NK$ positions are represented by the coloured squares: grey means that the positions are filled while white means that the positions are vacant. At each time step, following the output $\bm{\Pi}$ of MCTS, the player fills $\ell$ positions with value $1$ or $-1$.}
\label{SecII_MCTS}
\end{figure}

\textbf{Game-play with MCTS} -{}- The first part of the algorithm iteration in Fig. \ref{SecII_framework} is game-play with MCTS.
As illustrated in Fig. \ref{SecII_MCTS}, we play the game under the guidance of MCTS.
The upper half of Fig. \ref{SecII_MCTS} presents all the states in an episode, where squares represent the positions in the sequence set: grey squares mean that the position has already been filled (with value $1$ or $-1$); white squares mean that the position is still vacant (with value $0$). The initial state of each episode is an all-zero state $s_0$.
In state $s_i$, the player will follow a probabilistic policy $\bm{\Pi}(s_i)$ (not the raw policy output $\bm{P}$ of DNN) to choose $\ell$ symbols to fill in the next $\ell$ positions in the sequence set. This action yields a new state $s_{i+1}$.
The policy $\bm{\Pi}(s_i)$ is a distribution over the $2^\ell$ possible moves, and is given by MCTS.

The bottom half of Fig. \ref{SecII_MCTS} shows the MCTS process at each state $s_i$, where each circle (vertex) represents a possible state in the look-ahead search. In the MCTS for state $s_i$, we first set the root vertex $v_0$ to be $s_i$, and initialize a ``visited tree'' (this visited tree is used to record all the vertices visited in the MCTS. It is initialized to have only one root vertex). Look-ahead simulations are then performed along the visited tree starting at the root vertex.
Each simulation traces out a path of the visited tree, and terminates when an unseen vertex $v_L$ is encountered.
This unseen vertex will then be evaluated by DNN and added to the visited tree (i.e., a newly added vertex $v_L$ will be given the metric as $\psi_{\theta}(v_L)=(\bm{P}_L, \mathcal{R}'_L)$ to aid future simulations in evaluating which next move to select if the same vertex $v_L$ is visited again).
As more and more simulations are performed, the tree grows in size. The metric used in selecting next move for the vertices will also change (i.e., equations \eqref{EquAppA2} and \eqref{EquAppA3} in Appendix \ref{sec:AppA}) as the vertices are visited more and more in successive simulations. In a nutshell, estimated good vertices are visited frequently, while estimated bad vertices are visited rarely.
The resulting move-selection distribution at state $s_i$, i.e., $\bm{\Pi}(s_i)=(\pi_0,\pi_1,...,\pi_{2^\ell-1})$, is generated from the visiting counts of the root vertex's children in MCTS at states $s_i$ (i.e., equation \eqref{EquAppA4}).

Back to the upper part of Fig. \ref{SecII_MCTS}, after $\lceil NK/\ell \rceil$ time steps, the player obtains a complete sequence set $\mathcal{C}$  with metric value $\mathcal{M}(\mathcal{C})$ that gives a reward $\mathcal{R}(\mathcal{C})$. Then, we feed the $\mathcal{R}(\mathcal{C})$ to each state $s_i$ in this episode and store $(s_i,\bm{\Pi}(s_i),\mathcal{R})$ as an experience. One episode of game-play gives us $\lceil NK/\ell \rceil$ experiences.

\textbf{DNN update} -{}- The second part of the algorithm iteration in Fig. \ref{SecII_framework} is the training of the DNN based on the accumulated experiences over successive episodes.
First, from the description above, we know that MCTS is guided by DNN.
The capability of DNN determines the performance of MCTS since a better DNN yields more accurate evaluation of the vertices in MCTS.
In the extreme, if the DNN perfectly knows which sequence-set patterns are good and which are bad, then the MCTS will always head toward an optimal direction, hence the chosen moves are also optimal.
However, the fact is, DNN is randomly initialized, and its evaluation on vertices are quiet random and inaccurate initially.
Thus, our goal is to improve this DNN using the experiences generated from game-play with MCTS.

In the process of DNN update, the DNN is updated by learning the latest experiences accumulated in the game-play.
Given experience $(s_i,\bm{\Pi}(s_i),\mathcal{R})$ and $\psi_{\theta}(s_i)=(\bm{P}, \mathcal{R}')$,
1) the real reward $\mathcal{R}$ can be used to improve the value-function approximation $\mathcal{R}'$ of DNN;
2) the policy $\bm{\Pi}(s_i)$ given by MCTS at state $s_i$ can be used to improve the policy estimation $\bm{P}(s_i)$ of DNN\footnote{The policy $\bm{\Pi}(s_i)$ generated by MCTS is more powerful than the raw output $\bm{P}(s_i)$ of DNN \cite{GoZero2017}. Thus, $\bm{\Pi}(s_i)$ can be used to improve $\bm{P}(s_i)$.}.
Thus, the training process is to make $\bm{P}$ and $\mathcal{R}'$ more closely match $\bm{\Pi}$ and $\mathcal{R}$.

\vspace{1mm}
\noindent{\it \textbf{Remark:}
When we play games with MCTS to generate experiences, Dirichlet noise is added to the prior probability of root node $v_0$ to induce exploration, as that in AlphaGo \cite{GoZero2017}. These games are also called noisy games. Instead of noisy games, we can also play noiseless games in which Dirichlet noise is removed. Following the practice of AlphaGo, we play noisy games to generate the training experiences, but play noiseless games to evaluate the performance of AlphaSeq whose MCTS is guided by a particular trained DNN.
}
\vspace{1mm}

Overall, in one iteration, we
(i) play $G$ episodes of noisy games with $\psi_{\theta}$-guided MCTS to generate experiences, where $\psi_{\theta}$ is the current DNN;
(ii) use experiences gathered in the latest $z\times G$ episodes of games to train for a new DNN $\psi_{\theta'}$; (iii) assess the new DNN $\psi_{\theta'}$ by running $50$ noiseless games with $\psi_{\theta'}$-guided MCTS.
In the next iteration, we generate further experiences by playing $G$ episodes of noisy games with $\psi_{\theta'}$-guided MCTS.
Then these experiences are further used to train for yet another new DNN and so on and so forth.
The pseudocode for AlphaSeq is given in Table \ref{Table_algo}.

\begin{table}[t]
\centering
\caption{}
\label{Table_algo}
\includegraphics[scale=0.5]{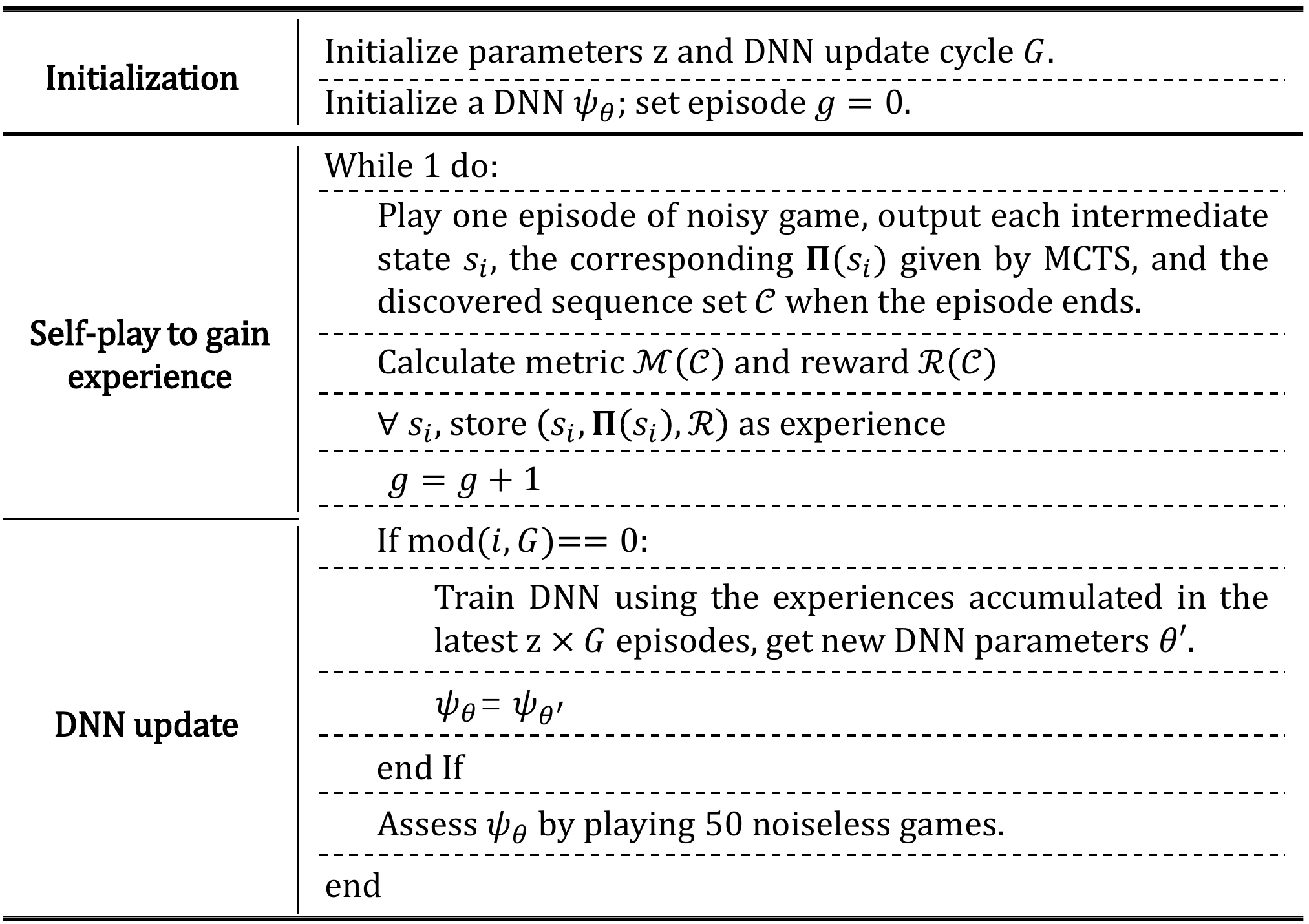}
\end{table}

In the following, we highlight some differences between AlphaSeq and AlphaGo.
\begin{itemize}
\item In AlphaGo, the total number of legal states is $\mathcal{O}(3^{NK})$ (in Go, $N=K=19$; each position can be occupied by no stones, a white stone, or a black stone). If we allow AlphaSeq to fill symbol positions in arbitrary order, then the complexity would be the same as AlphaGo in terms of the parameters $N$ and $K$. However, for AlphaSeq, we impose the order in which symbol positions are filled to reduce complexity. Now, the number of legal states reduces to\footnote{This maneuver is only applicable in our problem, but not in AlphaGo owing to the rule of the Go game.}
\begin{eqnarray}\label{EquA1}
\sum_{t=0}^{\lceil NK/\ell \rceil} 2^{t\ell} = \frac{2^{NK+\ell}-1}{2^\ell - 1},
\end{eqnarray}
that is, the state at the beginning of time step $t$ has $2^{tl}$ possible values. We found that imposing this restriction, while reducing complexity substantially, does not compromise the optimality of the sequence found.
\item In AlphaSeq, the choice of $\ell$ is a complexity tradeoff between MCTS and DNN; in AlphaGo, $\ell$ is always $1$.
As mentioned above, the universe of all states in the game forms a tree.
The depth of the tree is $\lceil NK/\ell \rceil$, which is the number of steps in Fig. \ref{SecII_MCTS} from left to right. This is exactly the number of MCTS we need to run in an episode. Thus, the larger the $\ell$, the fewer the MCTS we need to run. On the other hand, large $\ell$ yields more legal moves (i.e., $2^\ell$) in each state, hence burdening the DNN with a larger action space.
Overall, given $N$ and $K$, for small $\ell$, for example $\ell=1$, the mission of DNN is light since it only needs to determine to place $1$ or $-1$ in the next position. However, the number of MCTS we need to run in an episode is up to $NK$. In contrast, for large $\ell$, for example $\ell=K$, the number of MCTS we need to run in an episode is reduced to $N$, but the DNN is burdened with a heavier task because it needs to evaluate $2^K$ possible moves for each state.

\item In the game of Go, the board is invariant to rotation and reflection. Thus, we should augment the training data to let DNN learn these features. Specifically, in AlphaGo Zero, each experience (board state and move distribution) can be transformed by rotation and reflection to obtain extra training data, and the state in an experience is randomly transformed before the experience is fed to the DNN \cite{GoZero2017}. On the other hand, in our game, no rotation or reflection is required because all positions are predetermined. Any rotated or reflected state is an illegal state.

\item Compared with AlphaGo, our computational power is rather limited. Thus, for large sequence set beyond our computational power, a new technique, dubbed ¡°segmented induction¡±, is devised to progressively discover better sequence set. We exhibit in Section \ref{sec:PCR} that segmented induction performs well when applied to AlphaSeq.
\end{itemize}

In the following sections, we will demonstrate the searching capabilities of AlphaSeq in two applications:
in Section \ref{sec:MCCDMA}, we use AlphaSeq to rediscover an ideal complementary code set for multi-carrier CDMA systems;
in Section \ref{sec:PCR}, we use AlphaSeq to discover a new phase-coded sequence for pulse compression radar systems.

\section{Rediscover Ideal Complementary Code for Multi-Carrier CDMA}\label{sec:MCCDMA}
Code division multiple access (CDMA) is a multiple-access technique that enables numerous users to communicate in the same frequency band simultaneously \cite{CDMA1992}.
The fundamental principle of CDMA communications is to distinguish different users (or channels) by unique codes pre-assigned to them \cite{CDMA1991}. Thus, CDMA code design lies at the heart of CDMA technology.

\subsection{Codes in Legacy CDMA Systems}
Existing cellular CDMA systems work on a one-code-per-user basis \cite{CDMA1992}, \cite{NextCDMABook}. That is, the code set is designed such that exactly one code is assigned to each user, e.g., the orthogonal variable spreading factor (OVSF) code set used in W-CDMA downlink, the m-sequence set used in CDMA2000 uplink, and the Gold sequence set used in W-CDMA uplink \cite{OVSF}, \cite{CodeProceeding}.
However, legacy CDMA systems are self-jamming systems since their code sets cannot guarantee user orthogonality under practical constraints and considerations, such as user asynchronies, multipath effects, and random signs of consecutive bits of user data streams \cite{MCCDMA2001}\footnote{In CDMA, ``bit'' refers to the baseband modulated information symbols (only BPSK/QPSK modulated symbols are considered in this paper, in general it can be shown that the codes discussed in this section are applicable for higher-order modulations), while ``chip'' refers to the entries in the spread spectrum code. Thus, with respect to the nomenclature in Section \ref{sec:methodology}, ``chips'' in CDMA corresponds to ``symbol'' of a code sequence in Section \ref{sec:methodology}.}.

\begin{figure}[t]
  \centering
  \includegraphics[width=0.6\columnwidth]{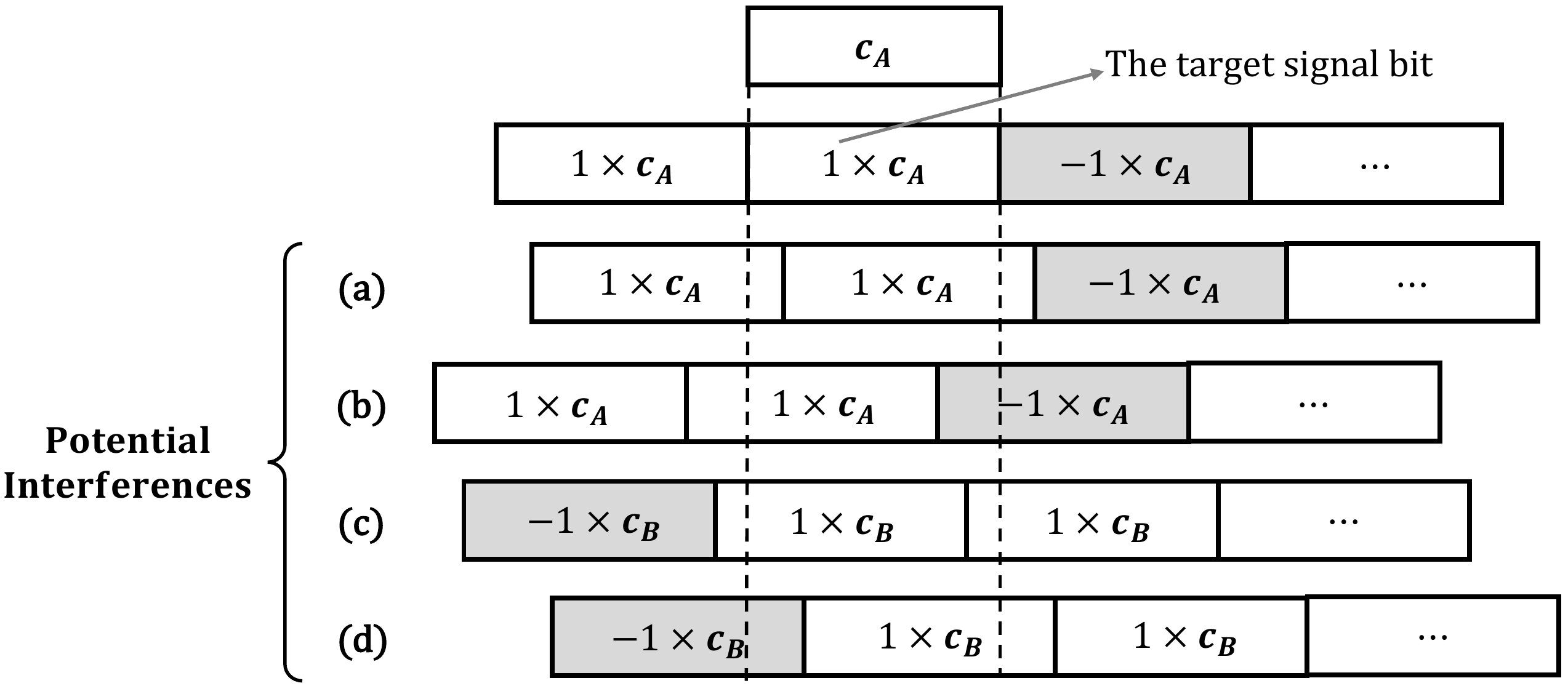}\\
  \caption{The interferences caused by user asynchronies (misalignments of bit boundaries), multi-paths, and random signs of consecutive bits, in CDMA uplink. To decode user A's data, the receiver correlates the received signal with code $\bm{c_A}$. Interferences are induced by (a) cyclic auto-correlation of $\bm{c_A}$; (b) flipped auto-correlation of $\bm{c_A}$; (c) cyclic cross-correlation between $\bm{c_A}$ and $\bm{c_B}$; (d) flipped cross-correlation between $\bm{c_A}$ and $\bm{c_B}$.}
\label{SecIII_Interference}
\end{figure}

In CDMA uplink, each user spreads its signal bits by modulating the assigned code, and the signals from multiple users overlap at the receiver.
To decode a user A's signal bit, as shown in Fig. \ref{SecIII_Interference}, the receiver cross-correlates the received signal with the locally generated code of user A.
However, due to user asynchronies, multi-paths, and random signs in consecutive bits, the correlation results can suffer from interferences introduced by multiple paths of user A's signal or signal from another user B.
The potential interferences can be computed by the correlations between the signal bit and two overlapping interfering bits:
when the signs of the two interfering bits are the same, the interferences are cyclic correlation functions (i.e., (a) and (c) in Fig. \ref{SecIII_Interference});
when the signs of the two interfering bits are different, the interferences are flipped correlation functions (i.e., (b) and (d) in Fig. \ref{SecIII_Interference}).
On the other hand, CDMA downlink is a synchronous CDMA system and there are no asynchronies among signals of different users. However, multi-path and random signs in consecutive bits can still cause interferences through the above correlations among codes.

Mathematically, it has been proven that the ideal one-code-per-user code set that simultaneously zero-forces the above correlation functions does not exist \cite{Welch}.
Code sets used in legacy CDMA systems trade-off among these correlation functions. For example, the m-sequence set has nearly ideal cyclic auto-correlation property (to be exact, the auto-correlation function of the m-sequence is $-1$ for any non-zero shift, hence is ``nearly'' optimal), while its cyclic cross-correlation and flipped correlation functions are unbounded.
The Gold sequence set and the Kasami sequence set (candidate in W-CDMA) have better cyclic cross-correlation properties and acceptable cyclic auto-correlation properties, but their flipped correlations are unbounded (see the excellent survey \cite{CodeProceeding} on the correlation functions of these sequences).

\subsection{Multi-Carrier CDMA and Ideal Complementary Codes}
The limitations of legacy CDMA systems motivate researchers to develop multi-carrier CDMA (MC-CDMA) systems where complementary codes can be used to simultaneously null all correlation functions among codes that may cause interferences \cite{MCCDMA2001}, \cite{MCCDMA2008}.

The basic idea of complementary codes is to assign a flock of $M$ element codes to each user, as opposed to just one code in legacy CDMA systems.
In MC-CDMA uplink, the signal bits of a user are spread by each of its $M$ element codes and sent over $M$ different subcarriers.
When passing through the channel, the $M$ subcarriers can be viewed as $M$ separate virtual channels that have the same delay. The receiver first de-spreads the received signal in each individual subcarrier (i.e., correlate the received signal in each sub-carrier with the corresponding element code), and sums up the de-spreading outcomes of all $M$ subcarriers. In other words, the operations in each individual channel are the same as legacy CDMA systems: the new step is the summing of the outputs of the $M$ virtual channels, which cancels out the interferences induced by individual correlations in the underlying subcarriers.

To be specific, let us consider a MC-CDMA system with $J$ users, where a flock of $M$ element codes of length $N$ are assigned to each user.
An ideal complementary code set $\mathcal{C}=\big\{c^m_j[n]:j=0,1,...,J\!-\!1;m=0,1,...,M\!-\!1;n=0,1,...,N\!-\!1\big\}$ that can enable interference-free MC-CDMA systems is a code set that meets the following criteria simultaneously:

\begin{enumerate}
\item Ideal cyclic auto-correlation function (CAF): for the $M$ element codes assigned to a user $j$, i.e., $\left\{\bm{c}^m_j:m=0,1,...,M-1 \right\}$, the sum of the cyclic auto-correlation function of each code is zero for any non-zero shift:
\begin{eqnarray}\label{EquIII1}
\textup{CAF}_j[v]=\sum_{m=0}^{M-1}\sum_{n=0}^{N-1}c^m_j[n]c^m_j[n+v]=0,
\end{eqnarray}
where delay (chip-level) $v=1,2,..,N-1$. Hereinafter, the index additions in the square brackets refer to modulo-$N$ additions.

\item Ideal cyclic cross-correlation function (CCF): for two flocks of codes assigned to users $j_1$ and $j_2$, i.e., $\left\{\bm{c}^m_{j_1},\bm{c}^m_{j_2}:m=0,1,...,M-1 \right\}$, the sum of their cyclic cross-correlation functions is always zero irrespective of the relative shift:
\begin{eqnarray}\label{EquIII2}
\textup{CCF}_{j_1,j_2}[v]=\sum_{m=0}^{M-1}\sum_{n=0}^{N-1}c^m_{j_1}[n]c^m_{j_2}[n+v]=0,
\end{eqnarray}
where delay $v=0,1,2,..,N-1$ and $j_1\neq j_2$.

\item Ideal flipped correlation function (FCF): for two flocks of codes assigned to users $j_1$ and $j_2$, i.e., $\left\{\bm{c}^m_{j_1},\bm{c}^m_{j_2}:m=0,1,...,M-1 \right\}$, the sum of their flipped correlation functions is always zero for any non-zero shift (flipped correlation is only defined for non-zero delay):
\begin{eqnarray}\label{EquIII3}
\textup{FCF}_{j_1,j_2}[v]=\sum_{m=0}^{M-1}\left\{ \sum_{n=0}^{N-v-1}c^m_{j_1}[n]c^m_{j_2}[n+v] - \sum_{n=N-v}^{N-1}c^m_{j_1}[n]c^m_{j_2}[n+v]  \right\} =0,
\end{eqnarray}
where delay $v=1,2,..,N-1$; $j_1$ and $j_2$ can be the same (flipped auto-correlation function) or different (flipped cross-correlation function).
\end{enumerate}

Some known mathematical constructions of ideal complementary codes are available in \cite{NextCDMABook}. In this section, we make use of AlphaSeq to rediscover a set of ideal complementary codes. Our aim is to investigate and evaluate the searching capability of AlphaSeq: i.e., whether it can rediscover an ideal complementary code set and how it goes about doing so. Further, we would like to investigate the impact of the hyper parameters used in the search algorithm on the overall performance of AlphaSeq, so as to obtain useful insights for discovering other unknown sequences (e.g., in Section \ref{sec:PCR}, we will make use of AlphaSeq to discover phase-coded sequences for pulse compression radar systems)

\subsection{AlphaSeq for MC-CDMA}
In this subsection, we use AlphaSeq to rediscover an ideal complementary code set for MC-CDMA systems.
As stated above, the ideal complementary code set is the code set that fulfills the three criteria in \eqref{EquIII1}, \eqref{EquIII2}, and \eqref{EquIII3}. In this context, given a sequence set $\mathcal{C}$, we define the following metric function to measure how good set $\mathcal{C}$ is for MC-CDMA systems.

\vspace{1mm}
\noindent\textbf{Metric Function:} For a sequence set $\mathcal{C}=\big\{c^m_j[n]:j=0,1,...,J\!-\!1;m=0,1,...,M\!-\!1;n=0,1,...,N\!-\!1\big\}$ consisting of $MJ$ sequences of the same length $N$, the metric function $\mathcal{M(C)}$ below reflects how good $\mathcal{C}$ is for MC-CDMA systems:
\begin{eqnarray}\label{MCCDMAA_Metric}
\mathcal{M}(\mathcal{C})=\sum_{j=0}^{J-1}\sum_{v=1}^{N-1}\big|\textup{CAF}_j[v]\big|
+\sum_{j_1=0}^{J-1}\sum_{j_2=j_1+1}^{J-1}\sum_{v=0}^{N-1}\big|\textup{CCF}_{j_1,j_2}[v] \big|
+\sum_{j_1=0}^{J-1}\sum_{j_2=j_1}^{J-1}\sum_{v=1}^{N-1}\big|\textup{FCF}_{j_1,j_2}[v]\big|,
\end{eqnarray}
Note that our desired metric value $\mathcal{M}^*=\inf \mathcal{M}(\mathcal{C})=0$.
For AlphaSeq, the objective is then to discover the sequence set that minimizes this metric function.

As an essential part of the training paradigm in AlphaSeq, a reward function is needed to map a found sequence set $\mathcal{C}$ to a reward $\mathcal{R}(\mathcal{C})$.
In general, we could design this reward function to be a linear (or non-linear) mapping from the value range of the metric function to the interval $[-1,1]$.
This is in fact a normalization process to fit general objectives to the architecture of AlphaSeq (specifically, normalizing the rewards of different problems allow these problems to share the same underlying hyper parameters in DNN and MCTS of the AlphaSeq architecture).
To rediscover the ideal complementary code, we define the reward function as follows:

\vspace{1mm}
\noindent\textbf{Reward Function:} For any sequence set $\mathcal{C}$ with metric $\mathcal{M}(\mathcal{C})$, the reward $\mathcal{R}(\mathcal{C})$ for MC-CDMA systems is defined as
\begin{equation}\label{MCCDMAA_Reward}
\mathcal{R(C)}=\left\{
\begin{array}{ccl}
1-\frac{2\mathcal{M}(\mathcal{C})}{\mathcal{M}_u}   &   \!\!\!\!\!\!\!\!   & {\textup{If}~0\leq\mathcal{M}(\mathcal{C})\leq\mathcal{M}_u,} \\
-1,                    &    \!\!\!\!\!\!\!\!  & {\textup{If}~\mathcal{M}(\mathcal{C}) > \mathcal{M}_u,}
\end{array} \right.
\end{equation}
where $\mathcal{M}_u$ is some sort of a worst-case $\mathcal{M}(\mathcal{C})$.
That is, when $\mathcal{M}(\mathcal{C})=\mathcal{M}_u$, then $\mathcal{R}(\mathcal{C})=-1$; and when $\mathcal{M}(\mathcal{C})=0$, then $\mathcal{R}(\mathcal{C})=1$.
We initially set $\mathcal{M}_u=\max_{\mathcal{C}}\mathcal{M}(\mathcal{C})$\footnote{See Appendix \ref{sec:AppB} for the derivation of $\max_{\mathcal{C}}\mathcal{M}(\mathcal{C})$.}, and initialize the DNN to $\psi_{\theta_0}$ (i.e., the parameters in the DNN is randomly set to $\theta_0$) to play $50$ noiseless games.
Then, $\mathcal{M}_u$ is set as the mean metric of the $50$ sequences found by these $50$ noiseless games, i.e., $\mathcal{M}_u=E[\mathcal{M}]$.
After this, $\mathcal{M}_u$ will not be changed anymore in future games.
We specify that the initial games do not find good sequences, but nevertheless the $50$ sequences yield an $E[\mathcal{M}]$ much lower than $\max_{\mathcal{C}}\mathcal{M}(\mathcal{C})$.
Using $E[\mathcal{M}]$ as $\mathcal{M}_u$ increases the slope of the first line in \eqref{MCCDMAA_Reward}.

Based on the metric function and reward function defined above, we implemented AlphaSeq and trained DNN to rediscover an ideal complementary code for MC-CDMA. A known ideal complementary code \cite{NextCDMABook} is chosen as benchmark.

\vspace{1mm}
\noindent\textbf{Benchmark:}
When $J=2$, $M=2$, and $N=8$, the ideal complementary code set exists. The mathematical constructions in \cite{NextCDMABook} gives us
\begin{equation}\label{SecIII_Bench}
\mathcal{C}_\textup{bench}=
\left(
\begin{array}{ccl}
\begin{bmatrix}
\begin{smallmatrix}
+1    & +1    & +1    & -1    & +1    & +1    & -1    & +1    \\[0.2cm]
+1    & -1    & +1    & +1    & +1    & -1    & -1    & -1
\end{smallmatrix}
\end{bmatrix}\\[2mm]
\begin{bmatrix}
\begin{smallmatrix}
+1    & +1    & +1    & -1    & -1    & -1    & +1    & -1    \\[0.2cm]
+1    & -1    & +1    & +1    & -1    & +1    & +1    & +1
\end{smallmatrix}
\end{bmatrix}
\end{array} \right)
\end{equation}
As can be seen, there are $J=2$ flocks of codes in $\mathcal{C}_\textup{bench}$, each flock contains $M=2$ codes and the length of each code is $N=8$. It can be verified that $\mathcal{M}(\mathcal{C}_\textup{bench})=0$.

To rediscover the code set, there are $32$ symbols to be filled in the game, and the number of all possible sequence-set patterns is $2^{32}\approx\mathcal{O}(10^9)$.
Discovering the global optimum out of $\mathcal{O}(10^9)$ possible patterns is in fact not a difficult problem based on brute-force exhaustive search (even though it takes several days on our computer).
The results of exhaustive search indicate that $\mathcal{C}_\textup{bench}$ in \eqref{SecIII_Bench} is not the only optimal pattern (that achieves $\mathcal{M}^*=0$) when $J=2$, $M=2$, and $N=8$.
There are in fact $384$ optimal patterns that can be divided into $12$ non-isomorphic types (i.e., each pattern has $31$ other isomorphic patterns, see Appendix \ref{sec:AppB} for the definition of isomorphic pattern).

\vspace{1mm}
\noindent\textbf{Implementation:}
We implemented and ran AlphaSeq on a computer with a single CPU (Intel Core i7-$6700$) and a single GPU (NVIDIA GeForce GTX $1080$ Ti)\footnote{Given the listed computation resource, another experiment is presented Appendix \ref{sec:AppD} to study the best found sequence versus time consumption in the RL process of AlphaSeq.}. The parameter settings are listed in Table. \ref{Table_paras1}.

\begin{table}[t]
\centering
\caption{}
\label{Table_paras1}
\includegraphics[scale=0.5]{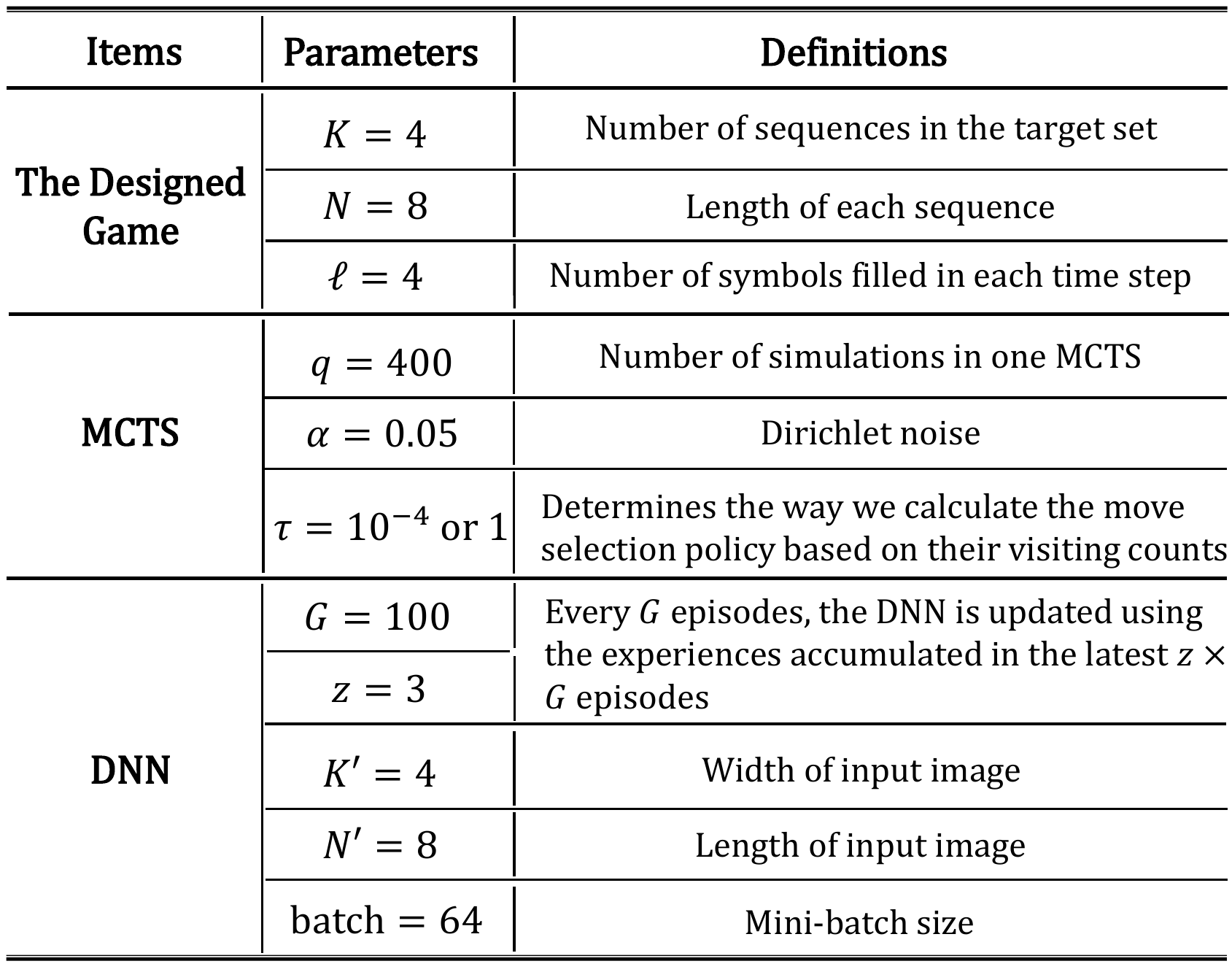}
\end{table}

For the symbol filling game, we set $K=MJ=4$, $N=8$, and $\ell=4$.
In other words, in each time step, $4$ symbols were placed in the $4\times 8$ sequence set, and an episode ended after $\lceil NK/\ell \rceil=8$ time steps when we obtained a complete sequence set.
The metric function and reward function were then calculated following \eqref{MCCDMAA_Metric} and \eqref{MCCDMAA_Reward}. An episode gave us $8$ experiences.

For DNN-guided MCTS, at each state $s_i$, we first set $s_i$ as the root node $v_0$, and then ran $q=400$ look-ahead simulations starting from $v_0$. For each simulation, Dirichlet noise $\textup{Dir}([\alpha_0,\alpha_1,...,\alpha_{2^\ell-1}])$ was added to the prior probability of $v_0$ to introduce exploration, where the parameters for Dirichlet distribution are set as $\alpha_0=\alpha_1=...=\alpha_{2^\ell-1}=\alpha=0.05$. After $400$ simulations, the probabilistic move-selection policy $\bm{\Pi}(s_i)$ was then calculated by \eqref{EquAppA4}, where we set $\tau=1$ for the first one third time steps (the probability of choosing a move is proportional to its visiting counts), and $\tau=10^{-4}$ for the rest of the time steps (deterministically choose the move with the most visiting counts).

The DNN implemented in AlphaSeq is a deep convolutional network (ConvNets). This DNN consists of six convolutional layers together with batch normalization and rectifier nonlinearities (detailed architecture of this ConvNets can be found in Appendix \ref{sec:AppA}). The DNN update cycle $G=100$ and $z=3$.
That is, every $G=100$ episodes, we trained the ConvNets using the experiences accumulated in the latest $z\times G=300$ episodes (i.e., $2400$ experiences) by stochastic gradient descent.
In particular, the mini-batch size was set to $64$, and we randomly sampled $\lceil 2400/64 \rceil$ mini-batches without replacement from the $2400$ experiences to train the ConvNets. For each mini-batch, the loss function is defined by \eqref{EquAppA5} in Appendix \ref{sec:AppA}.

\vspace{1mm}
\noindent{\it \textbf{Remark:} In Table \ref{Table_paras1}, the width and length of the input image fed into DNN is chosen to match with $N$ and $K$, i.e., $K'=K=4$ and $N'=N=8$.
However, it should be emphasized that this is not an absolute necessity.
In general, we find that setting the input of the DNN to be an $\ell\times\lceil NK/\ell \rceil$ image can speed up the learning process of DNN.
For example, if we had set $\ell=5$ instead of $\ell=4$ in this experiment, then it would better to set $K'=5$ and $N'=7$ (i.e., DNN takes an $5\times 7$ image as input, and in each time step, one row of the image is filled).
Accordingly, any intermediate state (i.e., a partially-filled $4\times 8$ sequence set pattern) must first be transformed to a $5\times 7$ image before it is fed into the ConvNets (the last $3$ symbols in the $5\times 7$ set will be padded with $0$ because the original $4\times 8$ set has $3$ fewer symbols).
}

\subsection{Performance Evaluation}
Over the course of training, AlphaSeq ran $8\times 10^3$ episodes, in which $6.4\times 10^4$ experiences were generated.
To monitor the evolution of AlphaSeq, every $G=100$ episodes when the DNN was updated, we evaluated the searching capability of AlphaSeq by using it (with the updated DNN) to play $50$ noiseless games (these $50$ noiseless games are in addition to the $G=100$ noisy games used to provide experiences to train the DNN).
The mean metric $E[\mathcal{M}]$ and the minimum metric $\min[\mathcal{M}]$ of the found $50$ sequence sets were recorded and plotted in Fig. \ref{MCCDMA_Sim1}.

\begin{figure}[t]
  \centering
  \includegraphics[width=0.9\columnwidth]{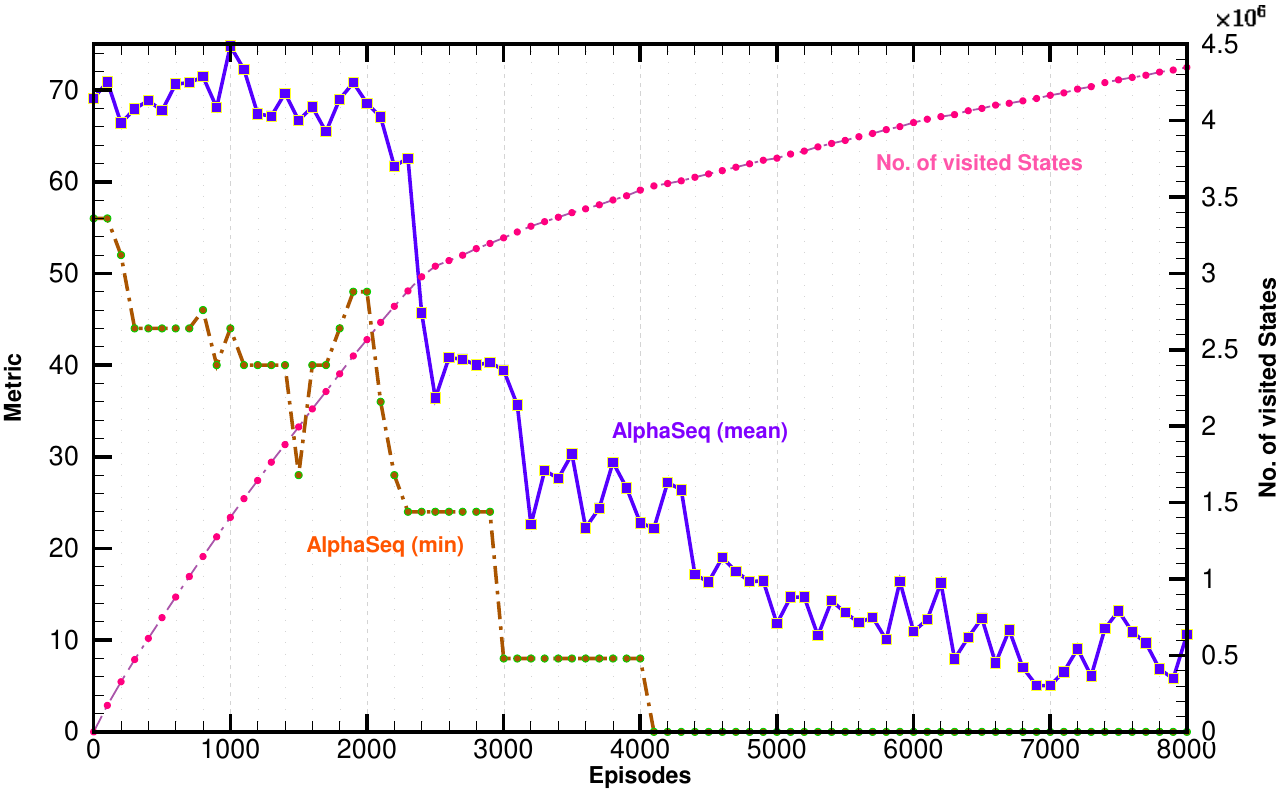}\\
  \caption{The reinforcement learning process of AlphaSeq to rediscover a set of ideal complementary code for MC-CDMA systems. Mean metric $E[\mathcal{M}]$, minimum metric $\min[\mathcal{M}]$ and the number of visited states versus episodes, where the DNN update cycle $G=100$ and $z=3$.}
\label{MCCDMA_Sim1}
\end{figure}

As can be seen from Fig. \ref{MCCDMA_Sim1}, with the continuous training of DNN, AlphaSeq gradually discovered sequence sets with smaller and smaller metric values. After $4100$ episodes, AlphaSeq rediscovered an ideal complementary code set $\mathcal{C}_\textup{alpha}$ given by
\begin{equation}\label{SecIII_Alpha1}
\mathcal{C}_\textup{alpha}=
\left(
\begin{array}{ccl}
\begin{bmatrix}
\begin{smallmatrix}
-1    & +1    & -1    & -1    & +1    & -1    & -1    & -1    \\[0.2cm]
-1    & -1    & -1    & +1    & +1    & +1    & -1    & +1
\end{smallmatrix}
\end{bmatrix}\\[2mm]
\begin{bmatrix}
\begin{smallmatrix}
+1    & -1    & +1    & +1    & +1    & -1    & -1    & -1    \\[0.2cm]
+1    & +1    & +1    & -1    & +1    & +1    & -1    & +1
\end{smallmatrix}
\end{bmatrix}
\end{array} \right)
\end{equation}

It is straightforward to see that $\mathcal{C}_\textup{alpha}$ is an isomorphic version to $\mathcal{C}_\textup{bench}$: i.e., if we denote $\mathcal{C}_\textup{bench}$ by $([\bm{a_1},\bm{a_2}]^T,[\bm{b_1},\bm{b_2}]^T)^T$, then $\mathcal{C}_\textup{alpha}=([-\bm{b_2},-\bm{b_1}]^T,[\bm{a_2},\bm{a_1}]^T)^T$.
We found that AlphaSeq could find different ideal sequence set in different runs. For example, in another run, AlphaSeq eventually discovered a non-isomorphic ideal sequence set to $\mathcal{C}_\textup{bench}$, giving
\begin{equation}\label{SecIII_Alpha2}
\mathcal{C}'_\textup{alpha}=
\left(
\begin{array}{ccl}
\begin{bmatrix}
\begin{smallmatrix}
+1    & -1    & -1    & -1    & -1    & -1    & +1    & -1    \\[0.2cm]
-1    & +1    & +1    & +1    & -1    & -1    & +1    & -1
\end{smallmatrix}
\end{bmatrix}\\[2mm]
\begin{bmatrix}
\begin{smallmatrix}
-1    & +1    & -1    & -1    & +1    & +1    & +1    & -1    \\[0.2cm]
+1    & -1    & +1    & +1    & +1    & +1    & +1    & -1
\end{smallmatrix}
\end{bmatrix}
\end{array} \right)
\end{equation}

The complexity of AlphaSeq is measured by means of distinct states that have been visited. Specifically, we stored all the states (including intermediate states and terminal states) encountered over the course of training in a Hash table. Every $G$ episodes, we recorded the length of the Hash table (i.e., the total number of visited states by then) and plotted them in Fig. \ref{MCCDMA_Sim1} as the training goes on.

An interesting observation is that, there is a turning point on the curve of the number of distinct visited states.
The slope of this curve corresponds to the extent to which AlphaSeq is exploring new states in its choice of actions.
Under the framework of AlphaSeq, there are two kinds of exploration:
1) Inherent exploration -{}- This is introduced by the variance of the action-selection policy. That is, the more random the action-selection policy is, the more new states are likely to be explored by AlphaSeq.
2) Artificial exploration -{}- We deliberately add extra artificial randomness to AlphaSeq to let it explore more states.
For example, the Dirichlet noise added to the root vertex in DNN-guided MCTS, the temperature parameter $\tau$ that determines how to calculate the policy all add to the randomness.
At the beginning of the game (i.e., episode $0$), the policy of AlphaSeq is quite random inherently because the DNN is randomly initialized.
Thus, both inherent exploration and artificial exploration contributes to the slope of this curve.
At the end of the game (i.e., episode $8\times 10^3$), the policy converges, hence the inherent exploration drops off, and only artificial exploration remains.

This turning point was in fact observed in all simulations of AlphaSeq in various applications we tried (not just the application for rediscovering complementary code here; see Section \ref{sec:PCR} on application of AlphaSeq to discover phase-coded sequences for pulse compression radar).
In general, we can then divide the overall reinforcement learning process of AlphaSeq into two phases based on this turning point. Phase I is an exploration-dominant phase (before the turning point), in which the behaviors of AlphaSeq are quite random.
As a result, AlphaSeq actively explores increasingly more states per $G$ episodes in the overall solution space.
After gaining familiarity with the whole solution space, AlphaSeq enters an exploitation-dominant phase (after the turning point), in which instead of exploring for more states, AlphaSeq tends to focus more on exploitation.

\vspace{1mm}
\noindent{\it \textbf{Remark:}
The DNN update cycle $G$ is important to guarantee that the algorithmic iteration proceeds in a direction of performance improvement.
In AlphaSeq, given a DNN $\psi_\theta$, the move-selection policy $\bm{\Pi}$ given by the $\psi_\theta$-guided MCTS is usually much stronger than the raw policy output $\bm{P}$ of $\psi_\theta$.
Thus, we first run $\psi_\theta$-guided MCTS to play $G$ games and generate $\lceil NK/\ell \rceil\times G$ experiences. Then, we use these experiences to train a new DNN $\psi_{\theta'}$, so that $\psi_{\theta'}$ can learn the stronger move given by $\psi_{\theta}$-guided MCTS.

In this context, the DNN update cycle $G$ must be chosen so that the $\lceil NK/\ell \rceil\times G$ experiences are sufficient to capture the fine details of $\bm{\Pi}$ given by $\psi_{\theta}$-guided MCTS.
In particular, parameter $G$ is closely related to $\ell$: a larger $\ell$ means more elements in $\bm{\Pi}$ (i.e., $\bm{\Pi}$ must capture $2^\ell$ possible moves in each step), and hence a larger $G$ is needed to guarantee that $\bm{\Pi}$ is well represented by the $\lceil NK/\ell \rceil\times G$ experiences.
}
\begin{figure}[t]
  \centering
  \includegraphics[width=0.7\columnwidth]{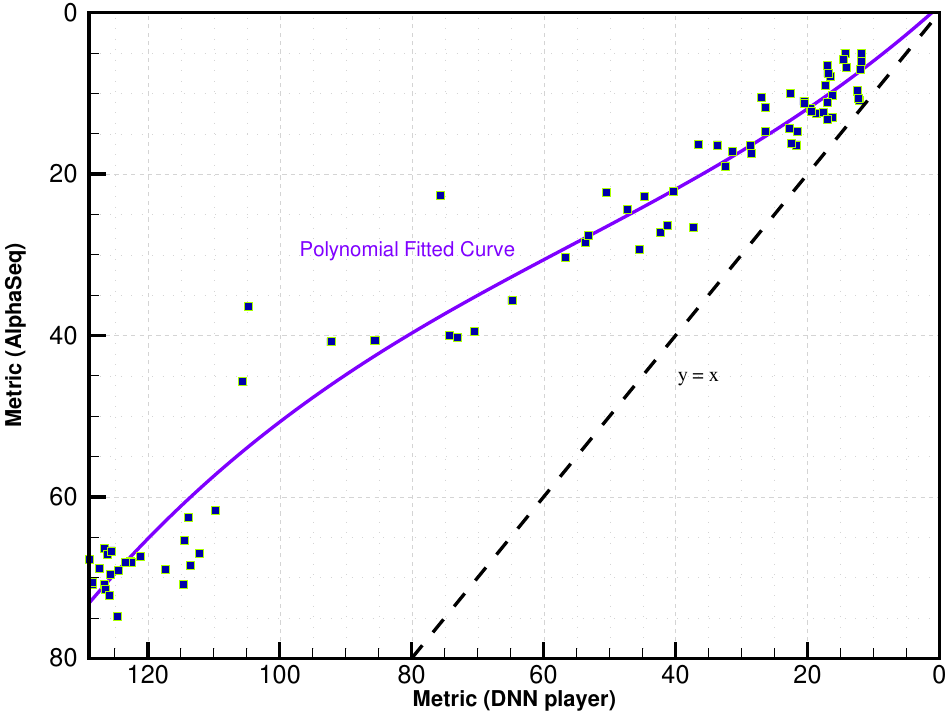}\\
  \caption{The polynomial fitted convergence curve for AlphaSeq and DNN player, where the DNN update cycle $G=100$. The positive direction of $x$-axis is a direction of performance improvement for DNN, while the positive direction of $y$-axis is a direction of performance improvement for AlphaSeq.}
\label{MCCDMA_Sim2}
\end{figure}

As stated in Section \ref{sec:methodology}, the essence of AlphaSeq is a process of iterative ``game-play with DNN-guided MCTS'' and ``DNN update'': the improvement of DNN brings about improvement of the DNN-guided MCTS, and the experiences generated by the improved MCTS in turn brings about further improvement of the DNN through training.
To verify this, each time when the DNN is updated, we assess the new DNN by using it (without MCTS, and no noise) to discover $50$ sequences and record their mean metric $E[\mathcal{M}']$. Specifically, at each state $s_i$, the player directly adopts the raw policy output of the DNN, i.e., $\bm{P}(s_i)$, to sample the next move without relying on the MCTS outputs $\bm{\Pi}(s_i)$.

Fig. \ref{MCCDMA_Sim2} presents all the $(E[\mathcal{M}'],E[\mathcal{M}])$ pair in the exploitation phase, and the corresponding polynomial fitted convergence curve. In particular, the positive direction of $x$-axis in Fig. \ref{MCCDMA_Sim2} is a direction of performance improvement for DNN, and the positive direction of $y$-axis is a direction of performance improvement for AlphaSeq. The convergence curve in Fig. \ref{MCCDMA_Sim2} reflects how the two ingredients, ``MCTS-guided game-play'' and ``DNN update'', interplay and mutually improve in the reinforcement learning process of AlphaSeq.

\section{AlphaSeq for Pulse Compression Radar}\label{sec:PCR}
Radar radiates radio pulses for the detection and location of reflecting objects \cite{RadarBook}.
A classical dilemma in radar systems arises from the choice of pulse duration:
given a constant power, longer pulses have higher energy, providing greater detection range;
shorter pulses, on the other hand, have larger bandwidth, yielding higher resolution.
Thus, there is a trade-off between distance and resolution.
Pulse compression radar can enable high-resolution detection over a large distance \cite{RadarBook}, \cite{PulseCodes}, \cite{misMatched3}.
The key is to use modulated pulses (e.g., phase-coded pulse) rather than conventional non-modulated pulses.

\subsection{Pulse Compression Radar and Phase codes}
The transmitter of a binary phased-coded pulse compression radar system transmits a pulse modulated by $N$ rectangular subpulses.
The subpulses are a binary phase code $\bm{s}$ of length $N$.
Each entry of the code is $+1$ or $-1$, corresponding to phase $0$ and $\pi$.
Following the definition in \cite{misMatched3} and \cite{misMatched2}, after subpulse-matched filtering and analog-to-digital conversion, the received sequence $\bm{y}$ is
\begin{eqnarray}\label{PCR_Signal}
\bm{y}=h_0\bm{s}+\sum_{n=1-N,n\neq 0}^{N-1}h_n\bm{J_n s} + \bm{w},
\end{eqnarray}
where
1) $\left\{h_n:n=1\!-\!N,2\!-\!N,...,N\!-\!2,N\!-\!1\right\}$ are coefficients proportional to the radar cross sections of different range bins \cite{misMatched3}. In particular, $h_0$ corresponds to the range bin of interest, and the radar's objective is to estimate $h_0$ given the received sequence $\bm{y}$;
2) $\bm{w}$ is the white Gaussian noise;
3) Matrix $\bm{J_n}$, as given in \eqref{SecIV_matricJ}, is a shift matrix capturing the different propagation time needed for the clutter to return from different range bins \cite{misMatched2}.
\begin{equation}\label{SecIV_matricJ}
\includegraphics[scale=0.35]{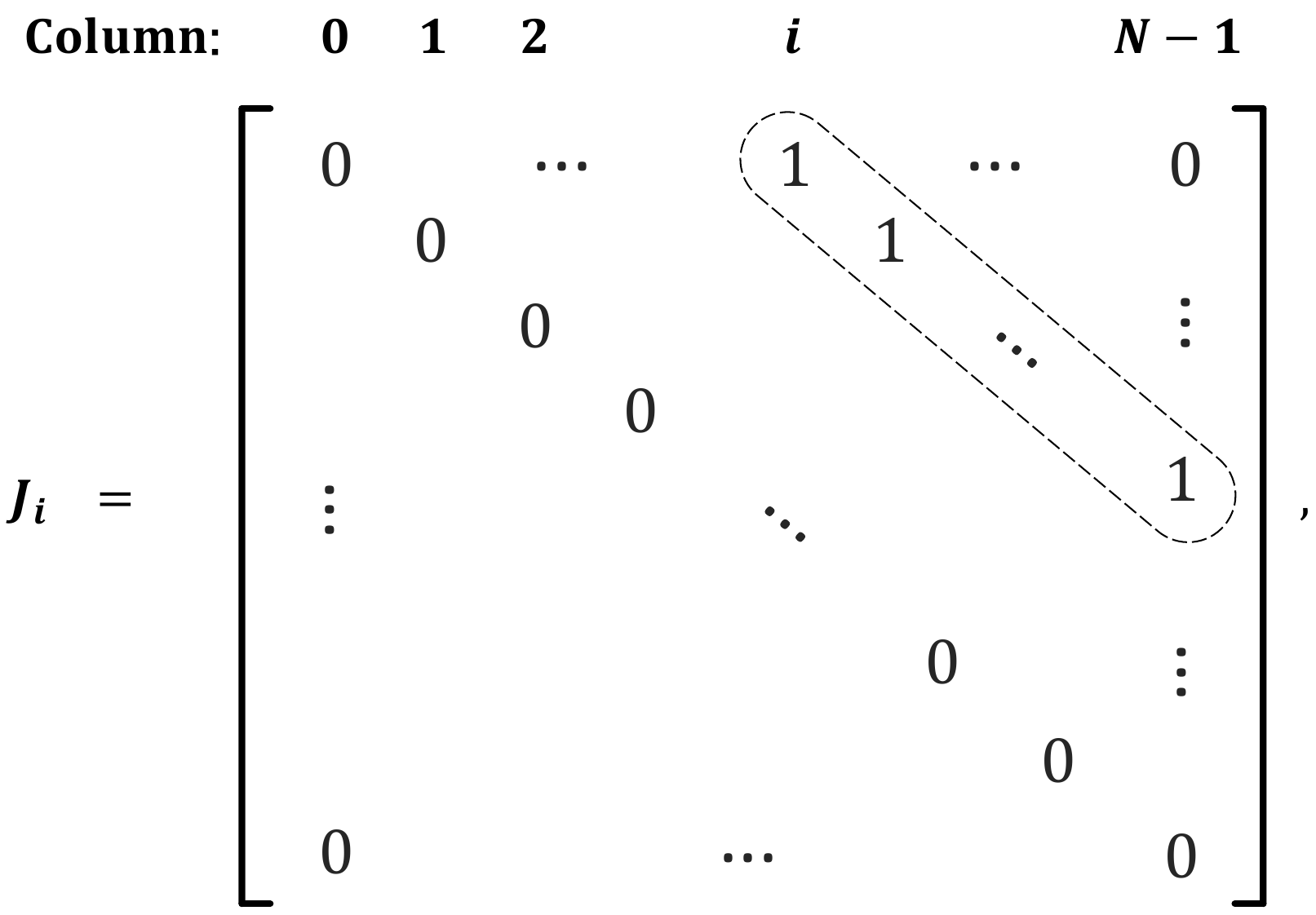}
\end{equation}
where $i=1,2,...,N\!-\!1$ and $\bm{J}_{-i}=\bm{J}^T_i$.
That is, in matrix $\bm{J_n}$, all entries except for that on the $n$-th off-diagonal are $0$.
The effect of matrix $\bm{J_n}$ is to right-shift or left-shift the phase code $\bm{s}$ with zero padding: when $n<0$, $\bm{J_n s}$ is a right-shifted version of $\bm{s}$; when $n>0$, $\bm{J_n s}$ is a left-shifted version of $\bm{s}$.

To estimate the coefficient $h_0$, a widely studied estimator is the matched filtering (MF) estimator:
\begin{eqnarray}
\widehat{h}_0=\frac{\bm{s}^T\bm{y}}{\bm{s}^T\bm{s}}=h_0+\sum^{N-1}_{n=1-N,n\neq 0}h_n\frac{\bm{s}^T\bm{J_n}\bm{s}}{\bm{s}^T\bm{s}}.
\end{eqnarray}
where the AWGN noise is ignored since the received signal is interference-limited (i.e., the interference power dominates over the noise power). Given the fact that we have no information on $\{h_n:n\neq 0\}$, the problem is then to discover a phase code $\bm{s}$ that can maximize the signal-to-interference ratio (SIR) $\gamma_\textup{MF}$ (larger SIR yields better estimation performance):
\begin{eqnarray}\label{MF}
\gamma_\textup{MF} = \frac{(\bm{s}^T\bm{s})^2}{\sum_{n=1-N,n\neq 0}^{N-1}(\bm{s}^T\bm{J_n}\bm{s})^2}.
\end{eqnarray}
In fact, this is the well-known ``merit factor problem'' occurring in various guises in many disciplines \cite{GolayMF}, \cite{HoholdtMF}, \cite{Whatcanbeused}.
In the past few decades, a variety of phase codes have been devised to achieve large SIR (merit factor), e.g., the Rudin-Shapiro sequences (asymptotically, $\gamma_\textup{MF}=3$), m-sequences (asymptotically, $\gamma_\textup{MF}=3$), and Legendre sequences (asymptotically, $\gamma_\textup{MF}=6$) (see the excellent surveys \cite{HoholdtMF}, \cite{Whatcanbeused} and the references therein).
Overall, the merit factor problem remains open. Experiment results show that $\gamma_\textup{MF}$ does not increase as the sequence length $N$ increases. So far, the best-known merit factor of $14.08$ is achieved by the Baker sequence of length $13$.

The motivation of the MF estimator comes from the fact that matched filtering provides the highest signal-to-noise ratio (SNR) in the presence of white Gaussian noise \cite{MF1994}.
However, in the case of Radar, the received signal is interference-limited, hence interference suppression is much more important. This motivates researchers to devise a mismatched filtering (MMF) estimator \cite{misMatched3},\cite{misMatched1}, \cite{misMatched2}.

Instead of using the transmitted phase-code $\bm{s}$, the MMF estimator uses a general real-valued code $\bm{x}$ to correlate the received sequence, giving
\begin{eqnarray}
\widehat{h}_0=\frac{\bm{x}^T\bm{y}}{\bm{x}^T\bm{s}}=h_0+\sum_{n=1-N,n\neq 0}^{N-1}h_n\frac{\bm{x}^T\bm{J_n}\bm{s}}{\bm{x}^T\bm{s}},
\end{eqnarray}
where the real-valued sequence $\bm{x}$ is to be optimized at the receiver.
The problem is then to find a pair of sequences $(\bm{s},\bm{x})$ so that the signal-to-interference ratio (SIR) $\gamma_\textup{MMF}$ in \eqref{MMF} can be maximized.
\begin{eqnarray}\label{MMF}
\gamma_\textup{MMF} =\frac{(\bm{x}^T\bm{s})^2}{\sum^{N-1}_{n=1-N,n\neq 0}(\bm{x}^T\bm{J_n}\bm{s})^2}.
\end{eqnarray}
It had been shown in \cite{misMatched2} that, given a phase code $\bm{s}$, the optimal sequence $\bm{x}$ that maximizes $\gamma_\textup{MMF}$ is $\bm{x}^*=\bm{R}^{-1}\bm{s}$, where matrix $\bm{R}$ is given by
\begin{eqnarray}\label{matrixR}
\bm{R}=\sum_{n=1-N,n\neq 0}^{N-1}\bm{J_n s}\bm{s}^T\bm{J_n}^T.
\end{eqnarray}
Substituting $\bm{x}^*=\bm{R}^{-1}\bm{s}$ in \eqref{MMF} gives
\begin{eqnarray}\label{MMF_Final}
\gamma_\textup{MMF} = \bm{s}^T\bm{R}^{-1}\bm{s}.
\end{eqnarray}
Notice that $\gamma_\textup{MMF}$ only depends on the phase code $\bm{s}$, hence, the objective for the design of the MMF estimator is then to discover a phase-code $\bm{s}$ that can maximize $\gamma_\textup{MMF}$ in \eqref{MMF_Final}.

\vspace{1mm}
\noindent{\it \textbf{Remark:}
The MMF estimator is superior to the MF estimator since $\gamma_\textup{MMF}$ is no less than $\gamma_\textup{MF}$ given the same phase code $\bm{s}$. However, the problem of discovering a phase code $\bm{s}$ that maximizes \eqref{MMF_Final} did not receive much attention from the research community compared with the merit factor problem (i.e., discovering a code $\bm{s}$ that maximizes \eqref{MMF}).
This is perhaps due to the more complex criterion, and the lack of suitable mathematical tools \cite{Whatcanbeused}.
}
In this section, we make use of AlphaSeq to discover phase codes for pulse compression radar with MMF estimator.

\subsection{AlphaSeq for Pulse Compression Radar}
We choose \eqref{MMF_Final} as the metric function of AlphaSeq:
\begin{eqnarray}\label{PCRmetric}
\mathcal{M}(\bm{s})=\bm{s}^T\bm{R}^{-1}\bm{s},
\end{eqnarray}
where matrix $\bm{R}$ is given in \eqref{matrixR}.
For AlphaSeq, the objective is to discover the sequence that can maximize this metric function.
Additional analyses on the structure of this metric function can be found in Appendix \ref{sec:AppC}.

Given a phase code $\bm{s}$ with metric $\mathcal{M}(\bm{s})$, we define a linear reward function as follows:
\begin{equation}\label{PCRreward}
\mathcal{R}(\bm{s})=
\frac{2\mathcal{M}(\bm{s}) - \mathcal{M}_u - \mathcal{M}_l}{\mathcal{M}_u-\mathcal{M}_l},
\end{equation}
where $\mathcal{M}_u$ and $\mathcal{M}_l$ are the upper and lower bounds for the search range of $\mathcal{M}(\bm{s})$.
In general, we could set $\mathcal{M}_u=\max_{\bm{s}}\mathcal{M}(\bm{s})$ and $\mathcal{M}_l=\min_{\bm{s}}\mathcal{M}(\bm{s})$.

\vspace{1mm}
\noindent{\it \textbf{Remark:}
In Appendix C, the value ranges of $\mathcal{M}(\bm{s})$ are derived as $\max_{\bm{s}}\mathcal{M}(\bm{s})=37$ and
$\min_{\bm{s}}\mathcal{M}(\bm{s})=\frac{16}{9N^3}$.
However, we empirically find that, if we directly set the search range $\mathcal{M}_u=37$ and $\mathcal{M}_l=\frac{16}{9N^3}$, AlphaSeq will be trapped in the exploration-dominant phase for a long time.
This is because $\mathcal{M}_u-\mathcal{M}_l$ is too large.
In other words, we are asking AlphaSeq to search over a large solution space for $\bm{s}$ all at once.
We will later introduce a technique dubbed ``segmented induction'' to induce AlphaSeq to zoom in to a good solution.
In essence, segmented induction uses a smaller range of $[\frac{16}{9N^3}, 37]$, but progressively changes $\mathcal{M}_u$ and $\mathcal{M}_l$ as better $\mathcal{M}(\bm{s})$ is obtained (i.e., focus our search within a subspace of $\bm{s}$ each time, but progressively changing the focus of the subspace within which we search).
}

Based on the metric function and reward function defined above, we implemented AlphaSeq and trained DNN to discover a phase code for the MMF estimator. A Legendre sequence \cite{Legendre} is chosen as the benchmark.

\vspace{1mm}
\noindent\textbf{Benchmark:} We choose the Legendre sequence of length $N=59$ as our benchmark:
\begin{equation}\label{Legendre}
\bm{s}_\textup{L}=
\begin{bmatrix}
\begin{smallmatrix}
+1 & +1 & -1 & +1 & +1 & -1 & +1 & -1 & +1 & -1 & -1 & -1 \\[0.2cm]
+1 & -1 & +1 & +1 & -1 & +1 & +1 & +1 & -1 & +1 & -1 & +1 \\[0.2cm]
-1 & -1 & +1 & -1 & -1 & +1 & +1 & +1 & -1 & +1 & +1 & +1 \\[0.2cm]
+1 & -1 & -1 & +1 & +1 & +1 & +1 & +1 & -1 & -1 & -1 & -1 \\[0.2cm]
-1 & +1 & +1 & -1 & -1 & -1 & -1 & +1 & -1 & -1 & -1 &
\end{smallmatrix}
\end{bmatrix}.\nonumber
\end{equation}

For the MMF estimator, this Legendre sequence yields SIR $\gamma_\textup{MMF}\approx 10.98$. For reference, $\bm{s}_\textup{L}$ yields a merit factor of $\gamma_\textup{MF}\approx 6.19$ when the MF estimator is used.

For the corresponding AlphaSeq game, there are $59$ symbols to fill.
The number of all possible sequence-set patterns is $2^{59}$.
The complexity of exhaustive search for the global optimum is $\mathcal{O}(10^{18})$, and it would take more than one million years for our computer to find the optimal solution.
In other words, the optimal solution of $\bm{s}$ when $N=59$ is unavailable.
In this context, the second benchmark we choose is random search.
For random search, we randomly create $59$-symbol sequences and record the maximum SIR obtained given a fixed budget of random trials.

\begin{table}[t]
\centering
\caption{}
\label{Table_paras2}
\includegraphics[scale=0.5]{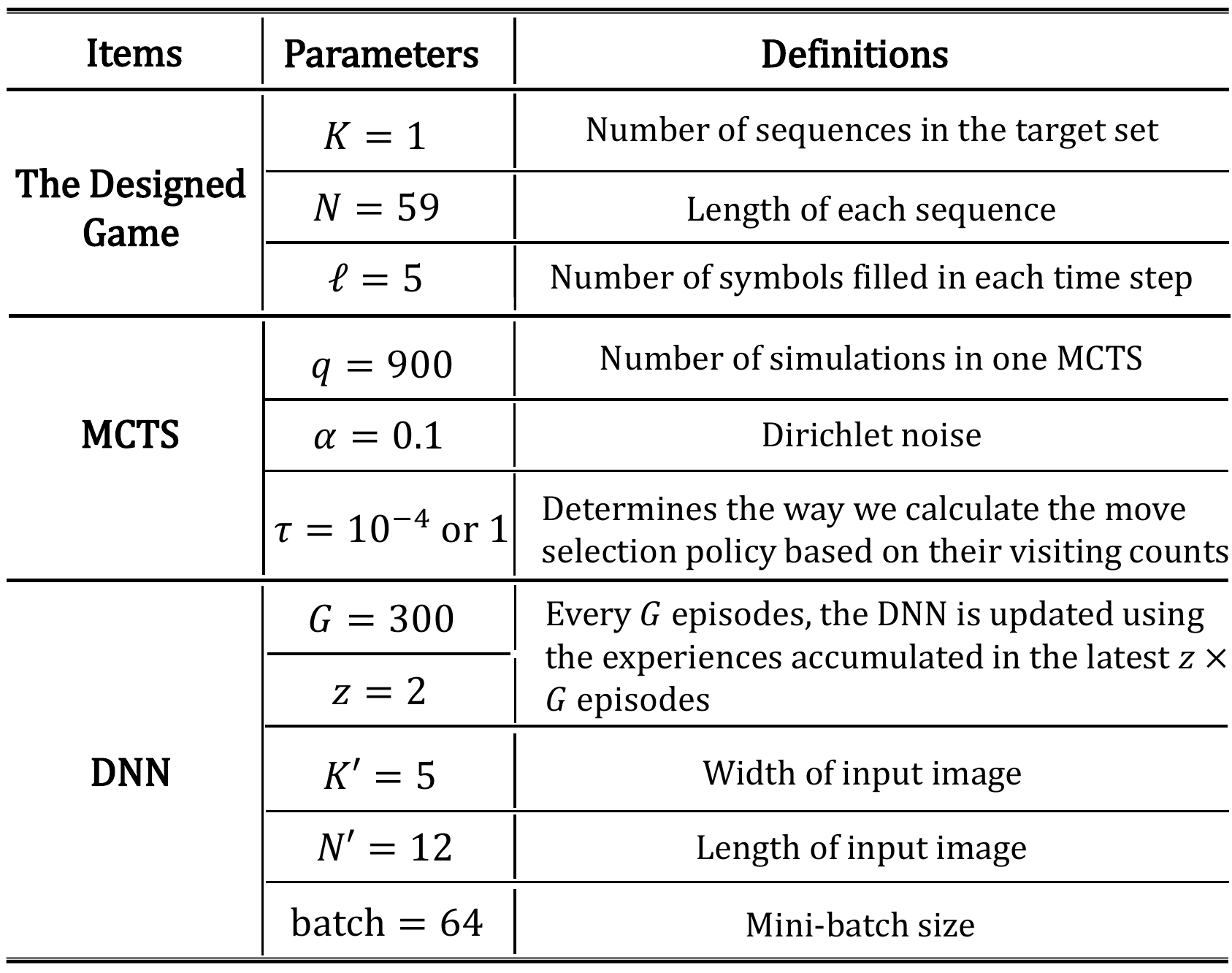}
\end{table}

\vspace{1mm}
\noindent\textbf{Implementation:} In the AlphaSeq implementation, the parameter settings are listed in Table \ref{Table_paras2}.
As seen in the table, we aim to discover one sequence of length $59$ wherein $K=1$ and $N=59$ in the AlphaSeq game.
The number of symbols filled in each time step is set to $\ell=5$, and the ConvNets takes $5\times 12$ images as input.
To feed an intermediate state (i.e., a partially-filled $1\times 59$ pattern) into the ConvNets, we first transform it to a $5\times 12$ image (the missing $1$ symbol will be padded with $0$).
A complete sequence is obtained after $\lceil NK/\ell \rceil=12$ time steps, where $60$ symbols are obtained.
Then, we ignore the last symbol and calculate the metric function and reward function following \eqref{PCRmetric} and \eqref{PCRreward}.
The DNN update cycle $G$ is set to $300$ and $z=2$. That is, every $G=300$ episodes, DNN will be updated using the experiences accumulated in the latest $600$ episodes.

Given the huge solution space, it is challenging for our computer to train AlphaSeq to find the optimal solution.
For one thing, each episode in this problem consumes much more time than the complementary code rediscovery problem in Section \ref{sec:MCCDMA}, because of the larger number of MCTSs run in each episode and the larger number of simulations run in each MCTS.
For another, the large solution space in this problem requires a massive number of exploration-dominant episodes so that AlphaSeq can visit enough number of states to gain familiarity with the whole solution space.
As a result, the exploration phase will last a long time before AlphaSeq enters the exploitation phase.
To tackle the above challenges, we use the follows two techniques to accelerate the training process:

\begin{enumerate}
\item Make more efficient use of experiences. Every $G$ episodes, we trained the DNN using the experiences accumulated in the latest $zG$ episodes ($zG\lceil NK/\ell \rceil$ experiences in total) by stochastic gradient descent. In section \ref{sec:MCCDMA}, the mini-batches were randomly sampled without replacement. That gave us $\lceil zG\lceil NK/\ell \rceil/64\rceil$ mini-batches ($64$ was the mini-batch size). Here, we want to make more efficient use of experiences. To this end, every $G$ episodes, we randomly sample $\lceil zG\lceil NK/\ell \rceil/64\rceil\times 6$ mini-batches with replacement from the latest $zG\lceil NK/\ell \rceil$ experiences to train the ConvNets.
\item Segmented induction. This technique is particularly useful when the upper and lower bounds of the metric function span a large range, or when there is no way to bound the metric function.
    The essence of segmented induction is to segment the large range of the metric function to several small ranges, and define the linear reward in small ranges rather than in a single large range.
    To be more specific, assuming a metric function with values within the range $[0,D]$.
    Then, rather than initializing $\mathcal{M}_l=0$ and $\mathcal{M}_u=D$ in \eqref{PCRreward}, we segment $[0,D]$ to three small overlapping ranges\footnote{a) Non-overlapping intervals are inadvisable. Experimental results show that AlphaSeq cannot learn well when using non-overlapping intervals. b) The small ranges segmented here are for illustration purpose only. In general, we need to design the ranges according to the specifics in different problems.} $[0,D/2]$, $[D/3,2D/3]$, $[D/2,D]$, and define the linear reward in these small ranges:
    in episode $0$, we define the reward function in the first small range and initialize $\mathcal{M}_l=0$ and $\mathcal{M}_u=D/2$.
    With the training of DNN, AlphaSeq is able to discover better and better sequences in the range $[0,D/2]$.
    When AlphaSeq discovers sequences with reward approaching $1$ (i.e., the mean metric function of the found sequences approaches $D/2$), we then redefine the reward with the second range $[D/3,2D/3]$.
    That is, we set $\mathcal{M}_l=D/3$, $\mathcal{M}_u=2D/3$, and let AlphaSeq discovers sequences in the second small range.
    When AlphaSeq is able to discover sequences with reward approaching $1$ again, we redefine the reward in the third small range, and so on and so forth.
    Overall, with a smaller range at a given time, the slope of the reward function in \eqref{PCRreward} increases, allowing AlphaSeq to distinguish the relative quality of different sequences with higher contrast.
\end{enumerate}

\subsection{Performance Evaluation}
For training, we ran AlphaSeq over $1.44\times 10^4$ episodes, generating $1.73\times 10^5$ experiences in total. As in Section \ref{sec:MCCDMA}, to monitor the evolution of AlphaSeq, every $G=300$ episodes when the DNN was updated, we evaluated the searching capability of AlphaSeq by using AlphaSeq (with the updated DNN) to play $50$ noiseless games, and recorded their mean metric $E[\mathcal{M}]$ and maximum metric $\max[\mathcal{M}]$.
Fig. \ref{FigPCR1} presents the $E[\mathcal{M}]$ and $\max[\mathcal{M}]$ versus episodes during the process of reinforcement learning.

\begin{figure}[t]
  \centering
  \includegraphics[width=0.9\columnwidth]{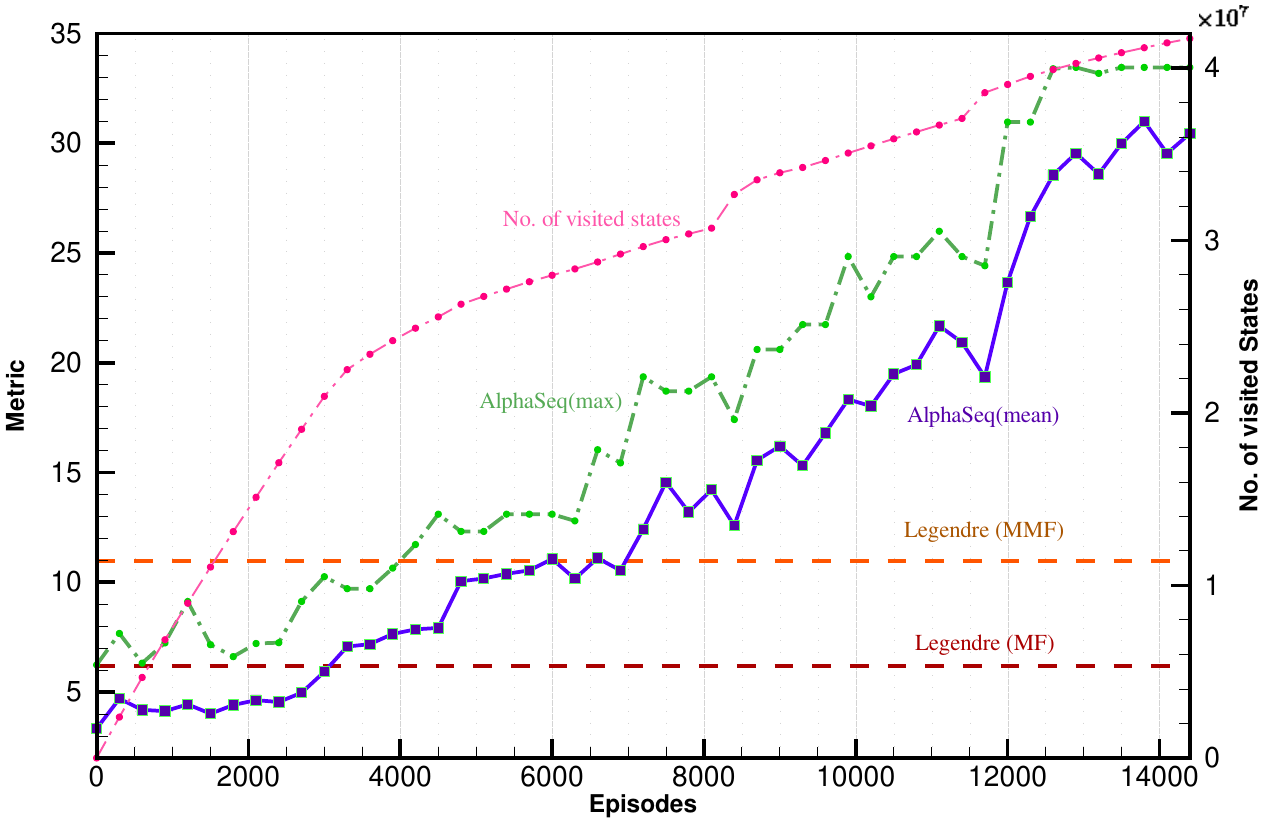}\\
  \caption{The reinforcement learning process of AlphaSeq to discover a phase-coded sequence for pulse compression radar. Mean metric $E[\mathcal{M}]$, maximum metric $\max[\mathcal{M}]$ and total number of visited states versus episodes, where the DNN update cycle $G=300$ and  $z=2$.}
\label{FigPCR1}
\end{figure}

As can be seen, the first $3300$ episodes are the exploration-dominant phase and the episodes after that are the exploitation-dominant phase.
After $1.26\times 10^4$ episodes, AlphaSeq discovers a sequence with metric $\mathcal{M}(\bm{s}_\textup{alpha})\approx 33.45$:
\begin{equation}
\bm{s}_\textup{alpha}=
\begin{bmatrix}
\begin{smallmatrix}
+1 & +1 & +1 & +1 & +1 & +1 & +1 & +1 & +1 & +1 & +1 & +1 \\[0.2cm]
+1 & +1 & +1 & +1 & -1 & -1 & -1 & -1 & -1 & -1 & -1 & -1 \\[0.2cm]
+1 & +1 & +1 & -1 & -1 & +1 & +1 & -1 & +1 & +1 & -1 & +1 \\[0.2cm]
-1 & -1 & +1 & -1 & +1 & -1 & +1 & -1 & -1 & +1 & -1 & +1 \\[0.2cm]
-1 & +1 & -1 & +1 & -1 & +1 & -1 & +1 & -1 & +1 & -1 &
\end{smallmatrix}
\end{bmatrix}.\nonumber
\end{equation}
Compared with the Legendre sequence, $\bm{s}_\textup{alpha}$ triples the SIR at the output of a MMF estimator.

\vspace{1mm}
\noindent{\it \textbf{Remark:}
In this implementation, the value range of $\mathcal{M}(\bm{s})$, i.e., $[16/9N^3,37]\approx [0,37]$, is segmented to three small ranges $[0,15]$, $[5,25]$, and $[10,37]$. In the first $8100$ episodes, the linear reward is defined in the first small range $[0,15]$: metric $0$ corresponds to reward $-1$, and $15$ corresponds to reward $1$; from episode $8101$ to $11400$, the linear reward is defined in the second small range $[5,25]$; after episode $11401$, the linear reward is defined in the last small range $[10,37]$.
}

\begin{figure}[t]
  \centering
  \includegraphics[width=0.7\columnwidth]{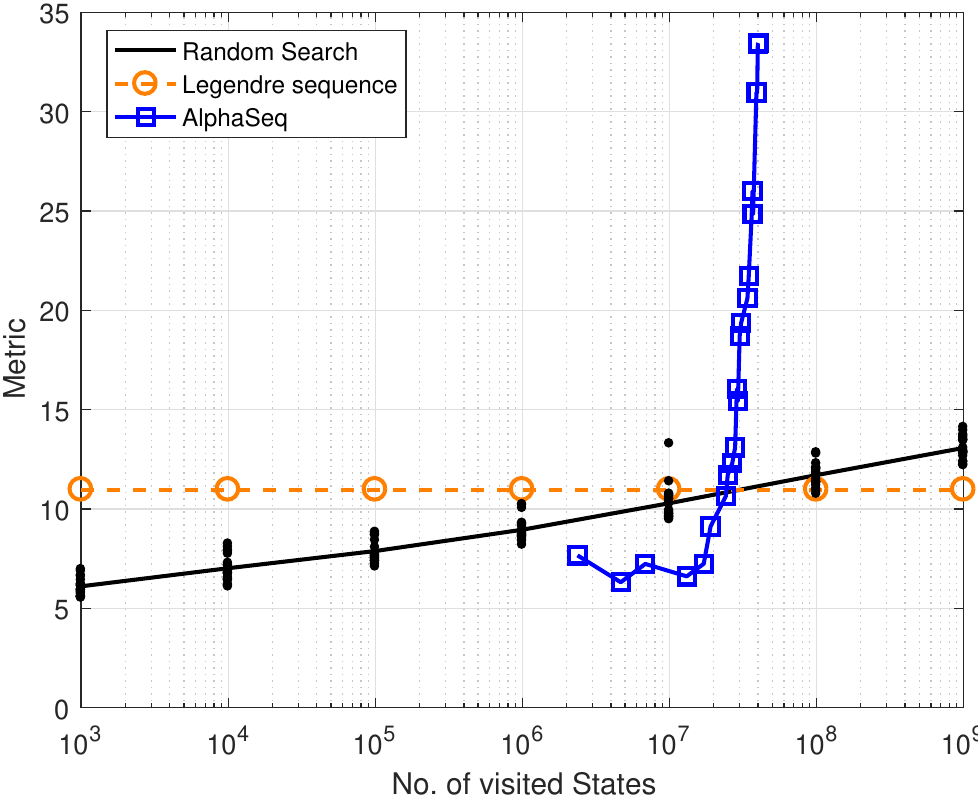}\\
  \caption{The searching capability comparison of AlphaSeq and random search. The AlphaSeq curve is the maximal metric $\max[\mathcal{M}]$ versus number of visited states.}
\label{FigPCR2}
\end{figure}

We next compare the searching capability of AlphaSeq with random search given the same complexity budget, where complexity is measured by the number of distinct visited states\footnote{Directly evaluating the complexity through the amounts of computation time consumed is not fair, because AlphaSeq uses both GPU and CPU in the implementation (and the CPU/GPU load varies over time) while random search uses only CPU (almost $100\%$).}. For AlphaSeq, the visited states include both intermediate states and terminal states, while for random search, only terminal states (i.e., completely-filled sequences) will be searched.

In Fig. \ref{FigPCR2}, the AlphaSeq curve is the maximal metric $\max[\mathcal{M}]$ versus number of visited states. This curve is a transcription of two curves in Fig. \ref{FigPCR1}: we combine the two curves, $\max[\mathcal{M}]$ versus episodes and number of visited states versus episodes, into one curve here. Fig. \ref{FigPCR2} also shows the maximal metric versus number of visited states for random search. To get this curve, given a state-visit budget, we performed $20$ runs of the experiments. For each run $i$, we traced the maximum metric value obtained after a given number of random trials, denoted by $\mathcal{M}^i_\textup{max}(n_v)$, where $n_v$ is the number of trials, which correspond to the number of visited (terminal) states. The black curve in Fig. \ref{FigPCR2} is $\frac{1}{20}\sum_i\mathcal{M}^i_\textup{max}(n_v)$ (i.e., a mean-max curve).

As can be seen from Fig. \ref{FigPCR2}, the largest metric that random search can find is on average a log-linear function of the number visited states. After randomly visiting $10^8$ states, the best sequence random search can find is on average with metric $11.71$. On the other hand, AlphaSeq discovers sequences with $\max[\mathcal{M}]=33.45$ after visiting only $4\times 10^7$ states.

\begin{figure}[t]
  \centering
  \includegraphics[width=0.7\columnwidth]{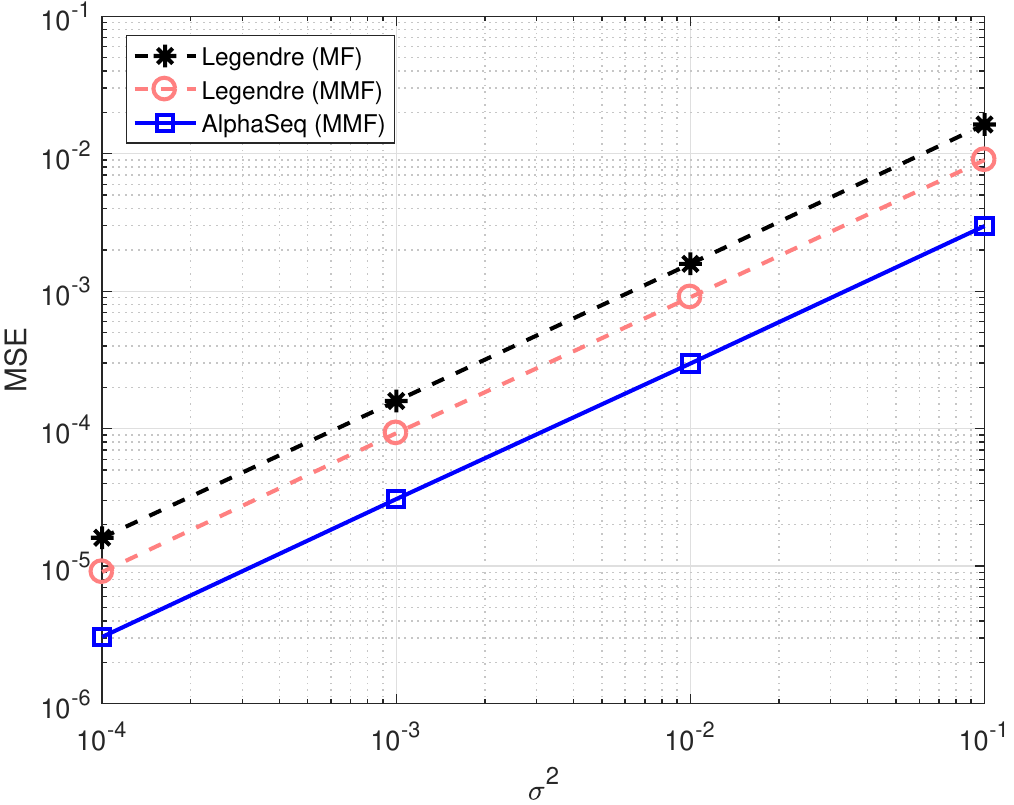}\\
  \caption{ The MSE of $\bm{s}_\textup{alpha}$ and $\bm{s}_\textup{L}$ for $h_0$ estimation in pulse compression radar systems.}
\label{FigPCR3}
\end{figure}

Finally, we assess the estimation performance of $\bm{s}_\textup{alpha}$ benchmarked against the Legendre sequence $\bm{s}_\textup{L}$ when used in a pulse compression radar system.
In the simulation, we assume the radar radiates pulse internally modulated by $\bm{s}_\textup{alpha}$ or $\bm{s}_\textup{L}$.
The received signal is given by equation \eqref{PCR_Signal}, where $\{h_n\}$ are Gaussian random variables with zero mean and same variance $\sigma^2$, and AWGN noise is ignored. The receiver estimates $h_0$ using an MMF estimator, and we measure the estimation performance by mean square error (MSE) $\epsilon=(h_0-\widehat{h_0})^2$.
Fig. \ref{FigPCR3} presents MSE versus $\sigma^2$ for $\bm{s}_\textup{alpha}$ and $\bm{s}_\textup{L}$. As can be seen, $\bm{s}_\textup{alpha}$ outperforms $\bm{s}_\textup{L}$, and the MSE gains are up to about $5.23$ dB.

\section{Conclusion}\label{sec:Conclusion}
This paper has demonstrated the power of deep reinforcement learning (DRL) for sequence discovery.
We believe that sequence discovery by DRL is a good supplement to sequence construction by mathematical tools, especially for problems with complex objectives intractable to mathematical analysis.

Our specific contributions and results are as follows:
\begin{enumerate}
\item We proposed a new DRL-based paradigm, AlphaSeq, to algorithmically discover a set of sequences with desired property. AlphaSeq leverages the DRL framework of AlphaGo to solve a Markov Decision Process (MDP) associated with the sequence discovery problem. The MDP is a symbol-filling game, where a player follows a policy to consecutively fill symbols in the vacant positions of a sequence set. In particular, AlphaSeq treats the intermediate states in the MDP as images, and makes use of deep neural network (DNN) to recognize them.

\item We introduced two new techniques in AlphaSeq to accelerate the training process. The first technique is to allow AlphaSeq to make $\ell$ moves at a time (i.e., filling $\ell$ sequence positions at a time). The choice of $\ell$ is a complexity tradeoff between the MCTS and the DNN. The second technique, dubbed ¡°segmented induction¡±, is to change the reward function progressively to guide AlphaSeq to good sequences in its learning process.

\item We demonstrated the searching capabilities of AlphaSeq in two applications: a) We used AlphaSeq to redicover a set of ideal complementary codes that can zero-force all potentially interferences in multi-carrier CDMA systems. b) We used AlphaSeq to discover new sequences that triple the signal-to-interference ratio -- benchmarked against the well-known Legendre sequence -- of a mismatched filter estimator in pulse compression radar systems. The mean square error (MSE) gains are up to $5.23$ dB for the estimation of radar cross sections.
\end{enumerate}

\appendices
\section{}\label{sec:AppA}
This appendix describes the implementation details of AlphaSeq.
Other than some custom features for our purpose, the general implementation follows AlphaGo Zero \cite{GoZero2017} and AlphaZero \cite{AlphaZero2017}. The source code can be found at GitHub \cite{sourcecode}.

\subsection{MCTS}
MCTS is performed at each intermediate state $s_i$ to determine policy $\bm{\Pi}(s_i)$, and this is achieved by multiple look-ahead simulations along the tree.
In the simulations, more promising vertices are visited frequently, while less promising vertices are visited less frequently.
The problem is how to determine which vertices are more promising and which are less promising in the simulations, i.e., how to evaluate a vertex in MCTS.
In standard MCTS algorithms, this vertex-evaluation is achieved by means of random rollouts.
That is, for a new vertex encountered in each simulation, we run random rollout from this vertex to a leaf vertex such that a reward can be obtained (see \cite{MCTS} for more details). The randomly sampled rewards over all simulations are then used to evaluate a vertex.

In AlphaGo/AlphaSeq, instead of random rollouts, DNN is introduced to evaluate a vertex.
The only two ingredients needed for MCTS are a root vertex $v_0$ and a DNN $\psi_{\theta}$.
First, given the root vertex $v_0$, a search tree can be constructed where each vertex contains $2^\ell$ edges (since there are $2^\ell$ possible moves for each state).
Each edge, denoted by $(v_i, a_j),~i=0,1,2,...,~j=0,1,2, ..., 2^\ell-1$, stores three statistics:
a visit count $N(v_i,a_j)$,
a mean reward $Q(v_i,a_j)$,
and an edge-selection prior probability $P(v_i,a_j)$.
Second, MCTS uses DNN $\psi_{\theta}$ to evaluate each vertex (state).
The input of $\psi_{\theta}$ is $v_i$ and the output is $(\bm{P}, \mathcal{R}')=\psi_{\theta}(v_i)$.
Specifically, each time we feed a vertex $v_i$ into the DNN, it outputs a policy estimation $\bm{P}$ and a reward estimation $\mathcal{R}'$.
Each entry in distribution $\bm{P}$ is exactly the prior probability $P(v_i,a_j)$ for each edge of vertex $v_i$, and $\mathcal{R}'$ will be used for updating the mean reward $Q(v_i,a_j)$, given by \eqref{EquAppA3} later.

MCTS is operated by means of look-ahead simulations.
Specifically, at a root vertex $v_0$, MCTS first initializes a ``visited tree'' (this visited tree is used to record all the vertices visited in the MCTS. It is initialized to have only one root vertex) and runs $q$ simulations on the visited tree.
Each simulation proceeds as follows \cite{GoZero2017}:
\begin{enumerate}
\item \textbf{Select} -{}- all the simulations start from the root vertex $v_0$ and finish when a vertex that has not been seen is encountered for the first time. During a simulation, we always choose the edge that yields a maximum upper confidence bound. Specifically, at each vertex $v_i$, the simulation selects edge $j^*$ to visit, and
\begin{equation}\label{EquAppA1}
j^* \! = \! \arg\max_j \!\left\{ Q(v_i,a_j) \! + \! c_p P(v_i,a_j)\frac{\sqrt{\sum_j\! N(v_i,a_j)}}{1\!+\!N(v_i,a_j)} \right\},\nonumber
\end{equation}
where $c_p$ is a constant controls the tradeoff between exploration and exploitation.
\item \textbf{Expand and Evaluate} -{}- when encountering a previously unseen vertex $v_L$ (for the first simulation, this $v_L$ is in fact $v_0$), the simulation evaluates it using DNN, giving, $(\bm{P}_L, \mathcal{R}'_L)=\psi_{\theta}(v_L)$, where the policy distribution $\bm{P}_L=\{P_L(j):j=0,1,2,...,2^\ell-1\}$. Then, we add this new vertex $v_L$ to the visited tree, and the statistics of $v_L$'s edges are initialized by $N(v_L,a_j)=0$, $Q(v_L,a_j)=0$, and $P(v_L,a_j)=P_L(j)$ for $j=0,1,2,...,2^\ell-1$.
\item \textbf{Backup} -{}-  After adding vertex $v_L$ to the visited tree, the simulation updates all the vertices along the trajectory of encountering $v_L$. Specifically, for each edge $(v_i, a_j)$ on the trajectory (including $v_L$), we update
\begin{eqnarray}\label{EquAppA2}
N(v_i,a_j) = N(v_i,a_j) + 1,
\end{eqnarray}
\begin{eqnarray}\label{EquAppA3}
Q(v_i,a_j) = Q(v_i,a_j) - \frac{Q(v_i,a_j)-\mathcal{R}'_L}{N(v_i,a_j)}.
\end{eqnarray}
\end{enumerate}

After $q$ simulations, MCTS then outputs a move selection probability for root vertex $v_0$ by
\begin{eqnarray}\label{EquAppA4}
\bm{\Pi}(v_0)=\textup{softmax}\left\{ \frac{1}{\tau}\log{N(v_0,a_j)}     \right\}.
\end{eqnarray}
That is, the move selection probability is determined by the visiting counts of the root vertex's edges.
Parameter $\tau$ is a temperature parameter as in AlphaGo Zero \cite{GoZero2017}.
In an episode, we set $\tau=1$ (i.e., the move-selection probability is proportional to the visiting counts of each edge, yielding more exploration) for the first one third time steps and $\tau=10^{-4}$ (deterministically choose the move that has the most visiting counts) for the rest of the time steps.

In the training iteration, when we play games to provide experiences for DNN, Dirichlet noise, i.e., $\textup{Dir}([\alpha_0,\alpha_1,...,\alpha_{2^\ell-1}])$ with positive real parameters $\alpha_0,\alpha_1,...,\alpha_{2^\ell-1}$, is added to the prior probability of root node $v_0$ to guarantee additional exploration.
Thus, these games are called noisy games.
Accordingly, there is noiseless games, in which Dirichlet noise is removed.
Usually, we play noiseless games to evaluate the performance of AlphaSeq with a trained DNN.

\subsection{DNN}
The DNN implemented in AlphaSeq is a deep convolutional network (ConvNets). This ConvNets consists of six convolutional layers together with batch normalization and rectifier nonlinearities, the details of which are shown in Fig. \ref{FigA1}.

\begin{figure}[t]
  \centering
  \includegraphics[width=0.9\columnwidth]{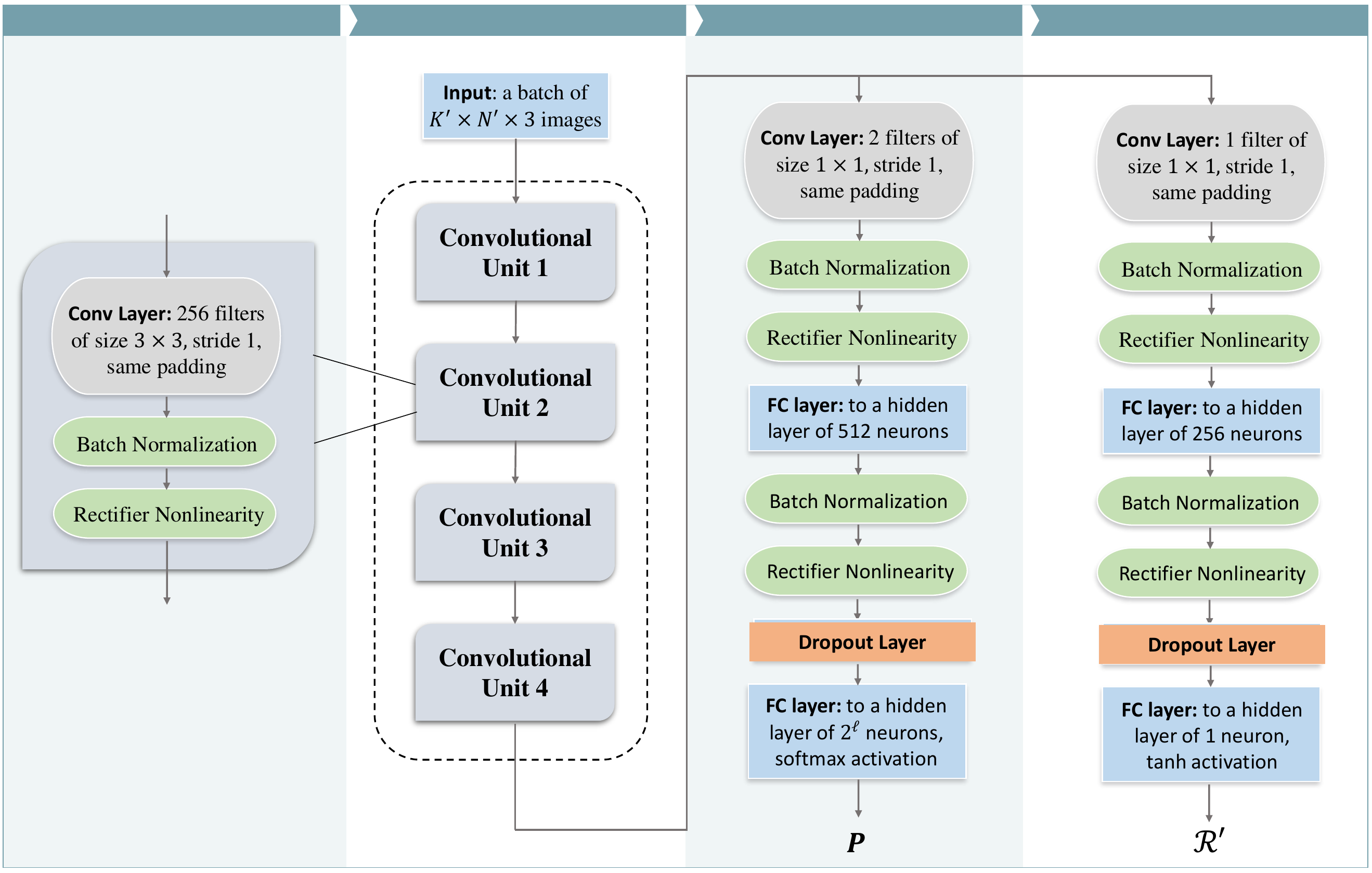}\\
  \caption{The deep ConvNets implemented in AlphaSeq. This ConvNets consists of six convolutional layers together with batch normalization and rectifier nonlinearities.}
\label{FigA1}
\end{figure}

\begin{itemize}
\item \textbf{Input} -{}- The ConvNets takes $K'\times N'\times 3$ image stack as input. For a state $s_i$ (i.e., an $K\times N$ partially-filled sequence-set pattern), we first transform it to a $K'\times N'\times 1$ image (in general we set $K'=\ell$ and $N'=\lceil NK/\ell \rceil$; zero-padding if $K'N'>KN$), and then perform feature extraction to transform it to a $K'\times N'\times 3$ image stack.
\item \textbf{Feature extraction} -{}- Feature extraction is a process to transform a $K'\times N'\times 1$ image to a $K'\times N'\times 3$ image stack comprising $3$ binary feature planes.
The three binary feature planes are constructed as follows.
The first plane, $X_1$, indicates the presentence of `1' in the $K'\times N'\times 1$ image: $X_1(i,j)=1$ if the intersection $(i,j)$ has value `1' in the $K'\times N'\times 1$ image, and $X_1(i,j)=0$ elsewhere.
The second plane, $X_2$, indicates the presentence of `-1' in the $K'\times N'\times 1$ image: $X_2(i,j)=1$ if the intersection $(i,j)$ has value `-1' in the $K'\times N'\times 1$ image, and $X_2(i,j)=0$ elsewhere.
The third plane, $X_3$, indicates the presentence of `0' in the $K'\times N'\times 1$ image: $X_3(i,j)=1$ if the intersection $(i,j)$ has value `0' in the $K'\times N'\times 1$ image, and $X_3(i,j)=0$ elsewhere.
\item \textbf{Output} -{}- For each state $s_i$, DNN will output a policy estimation (i.e., a probability distribution) $\bm{P}(s_i)$ $=$ ($p_0$, $p_1$, ..., $p_{2^\ell-1}$) as the prior probability for the $2^\ell$ edges of $s_i$, and a scalar estimation $\mathcal{R}'\in[-1,1]$ on the expected reward of $s_i$.
\item \textbf{Training} -{}- Every $G$ games, we use the experiences accumulated in the most recent $z\times G$ games
(i.e., $zG\lceil NK/\ell \rceil$ experiences) to update the DNN by stochastic gradient descent. The mini-batch size is set to $64$, and we randomly sample $\lceil zG\lceil NK/\ell \rceil/64\rceil$ mini-batches without replacement from the $zG\lceil NK/\ell \rceil$ experiences to train the ConvNets.
For each mini-batch, the loss function is defined to minimize the summation of mean-squared error and cross-entropy loss \cite{GoZero2017}
\begin{eqnarray}\label{EquAppA5}
\mathcal{L}=(\mathcal{R}-\mathcal{R}')^2 - \bm{\Pi}^T\log{\bm{P}} + c\left \| \theta \right \|^2.
\end{eqnarray}
where the last term is L2 regularization to prevent overfitting. Over the course of training, the learning rate is fixed to $10^{-4}$.
\end{itemize}

\section{}\label{sec:AppB}
\subsection{Supremum of the metric function}
This subsection derives the supremum of $\mathcal{M}(\mathcal{C})$ in \eqref{MCCDMAA_Metric} in Section \ref{sec:MCCDMA}.
Given the definitions of the correlation functions in \eqref{EquIII1}, \eqref{EquIII2}, and \eqref{EquIII3}, we first rewrite \eqref{MCCDMAA_Metric} as follows:
\begin{eqnarray}\label{EquAPPB1}
\mathcal{M}(\mathcal{C})
\!\!\!\!&=&\!\!\!\!
\sum_{j=0}^{J-1}\sum_{v=1}^{N-1}\big|\textup{CAF}_j[v]\big|
+\sum_{j_1=0}^{J-1}\sum_{j_2=j_1+1}^{J-1}\sum_{v=0}^{N-1}\big|\textup{CCF}_{j_1,j_2}[v] \big|
+\sum_{j_1=0}^{J-1}\sum_{j_2=j_1}^{J-1}\sum_{v=1}^{N-1}\big|\textup{FCF}_{j_1,j_2}[v]\big| \nonumber\\
&=&\!\!\!\!
\sum_{j_1=0}^{J-1}\sum_{j_2=j_1}^{J-1}\sum_{v=1}^{N-1}\Big(\big|\textup{CCF}_{j_1,j_2}[v]\big|+\big|\textup{PCF}_{j_1,j_2}[v]\big|\Big)
+\sum_{j_1=0}^{J-1}\sum_{j_2=j_1+1}^{J-1}\big|\textup{CCF}_{j_1,j_2}[0]\big|.
\end{eqnarray}

For the second term in \eqref{EquAPPB1}, we have
\begin{eqnarray}\label{EquAPPB2}
\sum_{j_1=0}^{J-1}\sum_{j_2=j_1+1}^{J-1}\big|\textup{CCF}_{j_1,j_2}[0]\big|\leq\frac{K(K-1)}{2}NM.
\end{eqnarray}

Moreover, the first term in \eqref{EquAPPB1} can be simplified as follows:
\begin{eqnarray}\label{EquAPPB3}
\sum_{j_1=0}^{J-1}\sum_{j_2=j_1}^{J-1}\sum_{v=1}^{N-1}\Big(\big|\textup{CCF}_{j_1,j_2}[v]\big|+\big|\textup{PCF}_{j_1,j_2}[v]\big|\Big)
\!\!\!\!&=&\!\!\!\!
\sum_{j_1=0}^{J-1}\sum_{j_2=j_1}^{J-1}\sum_{v=1}^{N-1}\Big(\big|\alpha[v]+\beta[v]\big|+\big|\alpha[v]-\beta[v]\big| \Big) \nonumber\\
&=&\!\!\!\!
\sum_{j_1=0}^{J-1}\sum_{j_2=j_1}^{J-1}\sum_{v=1}^{N-1}2\max\Big(\big|\alpha[v]\big|,\big|\beta[v]\big| \Big),
\end{eqnarray}
where
\begin{eqnarray}
\alpha[v]=\sum_{m=0}^{M-1}\sum_{n=0}^{N-v-1}c^m_{j_1}[n]c^m_{j_2}[n+v],~~~
\beta[v]=\sum_{m=0}^{M-1}\sum_{n=N-v}^{N-1}c^m_{j_1}[n]c^m_{j_2}[n+v].\nonumber
\end{eqnarray}

Notice that set $\mathcal{C}$ is binary, hence, $\alpha[v]$ and $\beta[v]$ are summations of $(N-v)M$ and $vM$ terms of $1$ or $-1$, respectively. As a result, we have
\begin{eqnarray}\label{EquAPPB4}
\sum_{v=1}^{N-1}2\max\Big(\big|\alpha[v]\big|,\big|\beta[v]\big| \Big)
\!\!\!\!&\leq&\!\!\!\! \sum_{v=1}^{N-1}2\max\Big\{ (N-v)M,vM \Big\} \nonumber\\
&=&\!\!\!\!\left\{
\begin{array}{lcl}
\frac{3N^2-4N+1}{2}M,       &    \!\!\!\!\!\!\!\!   & {\textup{for}~N~\textup{odd},}\\
\frac{3N^2-4N}{2}M,     &     \!\!\!\!\!\!\!\!  & {\textup{for}~N~\textup{even}.}
\end{array} \right.
\end{eqnarray}

Finally, we can bound $\mathcal{M(C)}$ from \eqref{EquAPPB2}, \eqref{EquAPPB3} and \eqref{EquAPPB4} as follows.

When $N$ is odd, it follows directly that
\begin{eqnarray}
\mathcal{M(C)}
\!\leq\!
\frac{K(K\!\!+\!1)}{2}\!\!\times\!\!\frac{3N^2\!\!-\!4N\!\!+\!1}{2}M\!+\!\frac{K(K\!\!-\!\!1)}{2}NM
\!=\!
\frac{(3N^2\!\!+\!\!1)(K^2\!\!+\!\!K)\!-\!2NK(K\!\!+\!3)}{4}M. \nonumber
\end{eqnarray}
When $N$ is even,
\begin{eqnarray}
\mathcal{M(C)}
\!\leq\!
\frac{K(K\!+\!1)}{2}\!\times\!\frac{3N^2\!-\!4N}{2}M\!+\!\frac{K(K\!-\!1)}{2}NM
\!=\!
\frac{3N^2(K^2\!+\!K)\!-\!2NK(K\!+\!3)}{4}M. \nonumber
\end{eqnarray}
These two upper bounds can be achieved by an all-$1$ (or an all-$-1$) sequence set, hence is tight.

In conclusion, we have
\begin{prop}
The supremum of $\mathcal{M(C)}$ is given by
\begin{equation}
\sup\mathcal{M(C)}=\left\{
\begin{array}{lcl}
\frac{(3N^2+1)(K^2+K)-2NK(K+3)}{4}M,       &    \!\!\!\!\!\!\!\!   & {\textup{for}~N~\textup{odd},}\\
\frac{3N^2(K^2+K)-2NK(K+3)}{4}M,     &     \!\!\!\!\!\!\!\!  & {\textup{for}~N~\textup{even}.}
\end{array} \right.\nonumber
\end{equation}
\end{prop}

\subsection{Isomorphic pattern}
When $J=2$, $M=2$, and $N=8$, each complementary code set has $31$ other isomorphic patterns.
For simplicity, let us denote a code set by $([\bm{a_1},\bm{a_2}]^T,[\bm{b_1},\bm{b_2}]^T)^T$, where $\bm{a_1}$ and $\bm{a_2}$ are the two element codes assigned to user A, while $\bm{b_1}$ and $\bm{b_2}$ are the two element codes assigned to user B.
Then, the following code-set patterns are isomorphic.
\begin{enumerate}
\item Row-switching.
\begin{eqnarray}
\left(
\begin{array}{ccl}
\begin{bmatrix}
\begin{smallmatrix}
\bm{a_1}   \\[0.2cm]
\bm{a_2}
\end{smallmatrix}
\end{bmatrix}\\[2mm]
\begin{bmatrix}
\begin{smallmatrix}
\bm{b_1}    \\[0.2cm]
\bm{b_2}
\end{smallmatrix}
\end{bmatrix}
\end{array} \right);~~
\left(
\begin{array}{ccl}
\begin{bmatrix}
\begin{smallmatrix}
\bm{b_1}   \\[0.2cm]
\bm{b_2}
\end{smallmatrix}
\end{bmatrix}\\[2mm]
\begin{bmatrix}
\begin{smallmatrix}
\bm{a_1}    \\[0.2cm]
\bm{a_2}
\end{smallmatrix}
\end{bmatrix}
\end{array} \right);~~
\left(
\begin{array}{ccl}
\begin{bmatrix}
\begin{smallmatrix}
\bm{a_2}   \\[0.2cm]
\bm{a_1}
\end{smallmatrix}
\end{bmatrix}\\[2mm]
\begin{bmatrix}
\begin{smallmatrix}
\bm{b_2}    \\[0.2cm]
\bm{b_1}
\end{smallmatrix}
\end{bmatrix}
\end{array} \right);~~
\left(
\begin{array}{ccl}
\begin{bmatrix}
\begin{smallmatrix}
\bm{b_2}   \\[0.2cm]
\bm{b_1}
\end{smallmatrix}
\end{bmatrix}\\[2mm]
\begin{bmatrix}
\begin{smallmatrix}
\bm{a_2}    \\[0.2cm]
\bm{a_1}
\end{smallmatrix}
\end{bmatrix}
\end{array} \right). \nonumber
\end{eqnarray}
\item Sign-reversing.
\begin{eqnarray}
\left(
\begin{array}{ccl}
\begin{bmatrix}
\begin{smallmatrix}
\bm{a_1}   \\[0.2cm]
\bm{a_2}
\end{smallmatrix}
\end{bmatrix}\\[2mm]
\begin{bmatrix}
\begin{smallmatrix}
\bm{b_1}    \\[0.2cm]
\bm{b_2}
\end{smallmatrix}
\end{bmatrix}
\end{array} \right);~~
\left(
\begin{array}{ccl}
\begin{bmatrix}
\begin{smallmatrix}
-\bm{a_1}   \\[0.2cm]
-\bm{a_2}
\end{smallmatrix}
\end{bmatrix}\\[2mm]
\begin{bmatrix}
\begin{smallmatrix}
\bm{b_1}    \\[0.2cm]
\bm{b_2}
\end{smallmatrix}
\end{bmatrix}
\end{array} \right);~~
\left(
\begin{array}{ccl}
\begin{bmatrix}
\begin{smallmatrix}
\bm{a_1}   \\[0.2cm]
\bm{a_2}
\end{smallmatrix}
\end{bmatrix}\\[2mm]
\begin{bmatrix}
\begin{smallmatrix}
-\bm{b_1}    \\[0.2cm]
-\bm{b_2}
\end{smallmatrix}
\end{bmatrix}
\end{array} \right);~~
\left(
\begin{array}{ccl}
\begin{bmatrix}
\begin{smallmatrix}
-\bm{a_1}   \\[0.2cm]
\bm{a_2}
\end{smallmatrix}
\end{bmatrix}\\[2mm]
\begin{bmatrix}
\begin{smallmatrix}
-\bm{b_1}    \\[0.2cm]
\bm{b_2}
\end{smallmatrix}
\end{bmatrix}
\end{array} \right); \nonumber\\
\left(
\begin{array}{ccl}
\begin{bmatrix}
\begin{smallmatrix}
\bm{a_1}   \\[0.2cm]
-\bm{a_2}
\end{smallmatrix}
\end{bmatrix}\\[2mm]
\begin{bmatrix}
\begin{smallmatrix}
\bm{b_1}    \\[0.2cm]
-\bm{b_2}
\end{smallmatrix}
\end{bmatrix}
\end{array} \right);~~
\left(
\begin{array}{ccl}
\begin{bmatrix}
\begin{smallmatrix}
\bm{a_1}   \\[0.2cm]
-\bm{a_2}
\end{smallmatrix}
\end{bmatrix}\\[2mm]
\begin{bmatrix}
\begin{smallmatrix}
-\bm{b_1}    \\[0.2cm]
\bm{b_2}
\end{smallmatrix}
\end{bmatrix}
\end{array} \right);~~
\left(
\begin{array}{ccl}
\begin{bmatrix}
\begin{smallmatrix}
-\bm{a_1}   \\[0.2cm]
\bm{a_2}
\end{smallmatrix}
\end{bmatrix}\\[2mm]
\begin{bmatrix}
\begin{smallmatrix}
\bm{b_1}    \\[0.2cm]
-\bm{b_2}
\end{smallmatrix}
\end{bmatrix}
\end{array} \right);~~
\left(
\begin{array}{ccl}
\begin{bmatrix}
\begin{smallmatrix}
-\bm{a_1}   \\[0.2cm]
-\bm{a_2}
\end{smallmatrix}
\end{bmatrix}\\[2mm]
\begin{bmatrix}
\begin{smallmatrix}
-\bm{b_1}    \\[0.2cm]
-\bm{b_2}
\end{smallmatrix}
\end{bmatrix}
\end{array} \right).\nonumber
\end{eqnarray}
Notice that $([\bm{a_1},-\bm{a_2}]^T,[-\bm{b_1},\bm{b_2}]^T)^T$ and $([-\bm{a_1},\bm{a_2}]^T,[\bm{b_1},-\bm{b_2}]^T)^T$ are isomorphic to $([\bm{a_1},\bm{a_2}]^T,[\bm{b_1},\bm{b_2}]^T)^T$. This is not straightforward at first sight. However, it can be seen that $([\bm{a_1},\bm{a_2}]^T,[\bm{b_1},\bm{b_2}]^T)^T$ is isomorphic to $([-\bm{a_1},\bm{a_2}]^T,[-\bm{b_1},\bm{b_2}]^T)^T$, while $([-\bm{a_1},\bm{a_2}]^T,[-\bm{b_1},\bm{b_2}]^T)^T$ is isomorphic to $([\bm{a_1},-\bm{a_2}]^T,[-\bm{b_1},\bm{b_2}]^T)^T$ and $([-\bm{a_1},\bm{a_2}]^T,[\bm{b_1},-\bm{b_2}]^T)^T$. Thus, $([\bm{a_1},\bm{a_2}]^T,[\bm{b_1},\bm{b_2}]^T)^T$ is also isomorphic to $([\bm{a_1},-\bm{a_2}]^T,[-\bm{b_1},\bm{b_2}]^T)^T$ and $([-\bm{a_1},\bm{a_2}]^T,[\bm{b_1},-\bm{b_2}]^T)^T$.
\end{enumerate}

Overall, each sequence-set pattern has $4\times8=32$ isomorphic patterns.

\section{}\label{sec:AppC}
This appendix analyzes the structure of the metric function $\mathcal{M}(\bm{s})$ in \eqref{PCRmetric} in Section \ref{sec:PCR}, and derives the upper and lower bounds for $\mathcal{M}(\bm{s})$. These two bounds are first derived mathematically. Then, we conjecture a tight upper bound for $\mathcal{M}(\bm{s})$. This conjectured bound can be achieved by a Barker sequence of length $13$.

\begin{lem}\label{LemmaC1}
Matrix $\bm{R}$ and $\bm{R}^{-1}$ are positive definite matrices.
\end{lem}

\begin{proof}
By equation \eqref{matrixR} in Section \ref{sec:PCR}, $\bm{R}$ is the summation of $2N-2$ symmetric matrices. Thus, $\bm{R}$ is symmetric, and hence $\bm{R}^{-1}$ is also symmetric ($\bm{R}$ is guaranteed to be non-singular \cite{misMatched2}).

A symmetric matrix $\bm{A}$ is said to be positive definite (semi-definite) if
1) the scalar $\bm{x^T Ax}$ is positive (nonnegative) for any non-zero column vector $\bm{x}$;
or 2) all the eigenvalues of $\bm{A}$ are positive (nonnegative).

Given \eqref{matrixR}, we have
\begin{eqnarray}\label{EquAppC1}
\bm{x^T Rx} =  \sum_{n=1-N,n\neq 0}^{N-1}\bm{x}^T\bm{J}_n \bm{s}\bm{s}^T\bm{J}_n^T \bm{x}\geq 0,
\end{eqnarray}
where the inequality follows since for each integer $n\in[1-N,N-1]$ and $n\neq 0$, $\bm{J}_n \bm{s}\bm{s}^T\bm{J}_n^T$ is positive semi-definite (In fact, for each individual  $n$, $\bm{J}_n \bm{s}\bm{s}^T\bm{J}_n^T$ has $N-1$ zero eigenvalues and one positive eigenvalue being $N-|n|$).

Given \eqref{EquAppC1}, $\bm{R}$ is positive semi-definite.
Moreover, $\bm{R}$ has no zero eigenvalue since $\bm{R}$ is nonsingular.
As a result, all the eigenvalues of $\bm{R}$ are positive, and hence $\bm{R}$ is guaranteed to be positive definite.
Further, all the eigenvalues of $\bm{R}^{-1}$ are also positive, hence $\bm{R}^{-1}$ is also positive definite.
\end{proof}

\begin{lem}[Weyl's inequalities \cite{Tao2001}]\label{LemmaC3}
Given two $N\times N$ Hermitian matrices $\bm{A}$ and $\bm{B}$, and their eigenvalues $\lambda_1\geq\lambda_2\geq...\geq\lambda_N$, $\mu_1\geq\mu_2\geq...\geq\mu_N$, respectively.
For another matrix $\bm{C}=\bm{A}+\bm{B}$, if we denote the eigenvalues of $\bm{C}$ by $v_1\geq v_2\geq...\geq v_N$. Then,
\begin{eqnarray}
\label{EquAppC2}
v_{i+j+1}\!\!\!& \leq &\!\!\! \lambda_{i+1}+\mu_{j+1}, \nonumber\\
\label{EquAppC3}
v_{n-i-j}\!\!\!& \geq &\!\!\! \lambda_{n-i}+\mu_{n-j}. \nonumber
\end{eqnarray}
\end{lem}

\begin{lem}[Singular value bound \cite{SingularBoundUpper}]\label{LemmaC4}
Let $\bm{A}$ be an $N\times M$ complex matrix, then its singular values are bounded by
\begin{eqnarray}\label{EquAppC4}
\sigma(\bm{A}) \leq \max_i{\sum_{j}|a_{i,j}|\sum_k{|a_{k,j}|}}.
\end{eqnarray}
where $i=0,1,...,N-1$, $j=0,1,...,M-1$, and $k=0,1,...,N-1$.
\end{lem}

\begin{lem}[Singular value bound \cite{SingularBound}]\label{LemmaC5}
Let $\bm{A}$ be an $N\times N$ complex matrix, its singular value is bounded by
\begin{eqnarray}\label{EquAppC5}
\sigma(\bm{A}) \geq \frac{|\textup{det}(\bm{A})|}{2^{N/2-1}\left \| \bm{A} \right \|_\textup{F}}.
\end{eqnarray}
\end{lem}

\begin{prop}
A lower bound and an upper bounds of $\mathcal{M}(\bm{s})$ in equation \eqref{PCRmetric} are given by
\begin{eqnarray}
\label{EquAppC6}
\frac{16}{9N^3} \leq \mathcal{M}(\bm{s}) \leq N^3 2^{N-4}. \nonumber
\end{eqnarray}
\end{prop}

\begin{proof}
Let us start from
\begin{eqnarray}\label{EquAppC8}
\lambda_\textup{min}(\bm{R}^{-1})\bm{s}^T\bm{s} \leq \bm{s}^T\bm{R}^{-1}\bm{s} \leq \lambda_\textup{max}(\bm{R}^{-1}) \bm{s}^T\bm{s},
\end{eqnarray}
where $\lambda_\textup{min}(\bm{R}^{-1})$ and $\lambda_\textup{max}(\bm{R}^{-1})$ are the minimum and maximum eigenvalues of $\bm{R}^{-1}$, and $\bm{s}^T\bm{s}=N$.
Moreover, given the fact that $\bm{R}$ and $\bm{R}^{-1}$ are positive definite (Lemma \ref{LemmaC1}), we have $\lambda_\textup{min}(\bm{R}^{-1})=1/\lambda_\textup{max}(\bm{R})$
and $\lambda_\textup{max}(\bm{R}^{-1})=1/\lambda_\textup{min}(\bm{R})$. Eq. \eqref{EquAppC8} can be rewritten as
\begin{eqnarray}\label{EquAppC9}
N/\lambda_\textup{max}(\bm{R}) \leq \bm{s}^T\bm{R}^{-1}\bm{s} \leq N/\lambda_\textup{min}(\bm{R}).
\end{eqnarray}
The problem is then to bound the maximum and minimum eigenvalues of matrix $\bm{R}$.

First, let us focus on the upper bound of $\lambda_\textup{max}(\bm{R})$.
Following the definition of matrix $\bm{R}$ in \eqref{matrixR}, we can factorize $\bm{R}$ as $\bm{R}=\bm{C}\bm{C}^T$, where $\bm{C}$ is given by
\begin{eqnarray}\label{EquAppC10}
\bm{C}
\!\!\!&=&\!\!\!
\big[\bm{J}_{N-1} \bm{s}~~\bm{J}_{N-2} \bm{s}~~\cdots~~\bm{J}_{2} \bm{s}~~\bm{J}_{1} \bm{s}~~\bm{J}_{-1} \bm{s}~~\bm{J}_{-2} \bm{s}~~\cdots
~~\bm{J}_{2-N} \bm{s}~~\bm{J}_{1-N} \bm{s}\big],
\end{eqnarray}
That is, the columns of $\bm{C}$ are constructed by $\bm{J}_{n} \bm{s}$ for each integer $n\in[1-N,N-1]$ and $n\neq 0$.

In particular, we can bound the maximum singular value of $\bm{C}$ by (Lemma \ref{LemmaC4})
\begin{eqnarray}\label{EquAppC11}
\sigma_\textup{max}(\bm{C}) \leq \frac{3}{4}N^2.
\end{eqnarray}

Given \eqref{EquAppC9} and \eqref{EquAppC11}, it follows directly that
\begin{eqnarray}\label{EquAppC12}
\mathcal{M}(\bm{s}) \!=\! \bm{s}^T\bm{R}^{-1}\bm{s} \!\geq\! N/\lambda_\textup{max}(\bm{R})\!=\! N/\sigma^2_\textup{max}(\bm{C})\geq \frac{16}{9N^3}.
\end{eqnarray}
This gives us a lower bound for $\mathcal{M}(\bm{s})$.

Next, we bound the minimum eigenvalues of $\bm{R}$.

The matrix $\bm{C}$ in \eqref{EquAppC10} can further be partitioned as $\bm{C}=\left[\bm{C_1}~\bm{C_2}\right]$,
where $\bm{C_1}$ is composed of $N$ columns of $\bm{C}$ and $\bm{C_2}$ contains the rest of the columns.
In this partition, we require that the $N$ columns in $\bm{C_1}$ are mutually independent. This requirement can be easily met, for example,
we could choose $\bm{C_1}$ to be
\begin{eqnarray}\label{EquAppC13}
\bm{C_1}\!\!\!&=&\!\!\!
\big[\bm{J}_{N-1} \bm{s}~~\bm{J}_{N-2} \bm{s}~~\cdots~~\bm{J}_{N-i} \bm{s}~~\bm{J}_{-i} \bm{s}~~\bm{J}_{-i-1} \bm{s} ~~\cdots
~~\bm{J}_{2-N} \bm{s}~~\bm{J}_{1-N} \bm{s}\big],
\end{eqnarray}
where $i$ can be $1,2,...,N-1$.

Given $\bm{C}=\left[\bm{C_1}~\bm{C_2}\right]$, we have $\bm{R}=\bm{C}\bm{C}^T=\bm{C_1}\bm{C_1}^T+\bm{C_2}\bm{C_2}^T$.

Lemma \ref{LemmaC3} implies the following fact:
\begin{eqnarray}\label{EquAppC14}
\lambda_\textup{min}(\bm{R}) \!\geq\! \lambda_\textup{min}(\bm{C_1}\bm{C_1}^T)\!+\!\lambda_\textup{min}(\bm{C_2}\bm{C_2}^T) \!=\! \lambda_\textup{min}(\bm{C_1}\bm{C_1}^T),
\end{eqnarray}
where the equality follows because $\bm{C_2}\bm{C_2}^T$ is singular.

Then, the problem is to bound the minimum eigenvalue of $\bm{C_1}\bm{C_1}^T$, or equivalently, the minimum singular value of $\bm{C_1}$.

By Lemma \ref{LemmaC5}, the minimum singular value of $\bm{C_1}$ can be bounded by
\begin{eqnarray}\label{EquAppC15}
\sigma_\textup{min}(\bm{C_1})\geq \frac{|\textup{det}(\bm{C}_1)|}{2^{N/2-1}\left \| \bm{C}_1 \right \|_\textup{F}}.\nonumber
\end{eqnarray}

The construction in \eqref{EquAppC13} suggests that the eigenvalues of $\bm{C_1}$ can only be $1$ or $-1$. In other words, $|\textup{det}(\bm{C_1})|=1$ and
\begin{eqnarray}\label{EquAppC16}
\sigma_\textup{min}(\bm{C_1})\geq \frac{|\textup{det}(\bm{C}_1)|}{2^{N/2-1}\left \| \bm{C}_1 \right \|_\textup{F}}=\frac{2^{2-N/2}}{N}.
\end{eqnarray}

Substituting \eqref{EquAppC16} into \eqref{EquAppC14}, yielding
\begin{eqnarray}\label{EquAppC17}
\lambda_\textup{min}(\bm{R})\geq\lambda_\textup{min}(\bm{C_1}\bm{C_1}^T)=\sigma^2_\textup{min}(\bm{C_1})\geq\frac{2^{4-N}}{N^2}.\nonumber
\end{eqnarray}

By \eqref{EquAppC9}, it follows directly that
\begin{eqnarray}\label{EquAppC18}
\mathcal{M}(\bm{s}) = \bm{s}^T\bm{R}^{-1}\bm{s} \leq N^3 2^{N-4}.
\end{eqnarray}
\end{proof}
To verify the tightness of the derived bounds, we then exhaustive search to find $\inf \mathcal{M}(\bm{s})$ and $\sup\mathcal{M}(\bm{s})$ when $N\leq 30$.
Simulation results indicate that the derived lower bound is close to $\inf \mathcal{M}(\bm{s})$, while the derived upper bound is far from tight. From the simulation, it seems that $\sup\mathcal{M}(\bm{s})$ is bounded by a constant value, rather than increasing as $N$ increases.
In the following, we conjecture a tight upper bound for $\mathcal{M}(\bm{s})$.

\begin{con}
For any phase code $\bm{s}$, a tight upper bound for $\mathcal{M}(\bm{s})$ is given by $\mathcal{M}(\bm{s})\leq 37$. The equality is achieved by a Barker sequence of length $13$, that is,
\begin{eqnarray}
\bm{s}_\textup{B}=[1,1,1,1,1,-1,-1,1,1,-1,1,-1,1]^T.
\end{eqnarray}

\end{con}
This conjectured bound is motivated from the study of the merit factor problem in the recent decades \cite{HoholdtMF}.
Until now, Baker sequence of length $13$ is the best known sequence that yields the maximal merit factor \cite{Whatcanbeused}.
For Barker sequence, a favorable property is that $|\bm{s}_\textup{B}^T\bm{J_0}\bm{s}_\textup{B}|=N$, while $|\bm{s}_\textup{B}^T\bm{J_n}\bm{s}_\textup{B}|=1$ for $n\neq 0$.
This is why we conjecture that a Barker sequence\footnote{Another famous conjecture is that there is no Barker sequence of length $N>13$. Although not being proven, it had been shown that there is no Barker sequence of length $13<N<10^{22}$ \cite{Whatcanbeused}.} can yield the maximal $\mathcal{M}(\bm{s})$.

\section{}\label{sec:AppD}
Given our computation resources (a CPU: Intel Core i7-6700), and a GPU: NVIDIA GeForce GTX 1080 Ti), this appendix studies the best found sequence versus time consumption in the learning process of AlphaSeq.
We let AlphaSeq discover a sequence of length $N=28$ for pulse compression radar (the length $N=28$ is chosen so that the global optimal is accessible by exhaustive search. We could then measure how long it takes AlphaSeq to reach the global optimal). As shown in Section \ref{sec:PCR}, the objective is to discover a sequence $\bm{s}$ that can maximize the metric $\mathcal{M}(\bm{s})=\bm{s}^T\bm{R}^{-1}\bm{s}$ (i.e., equation \eqref{PCRmetric}).

\noindent\textbf{Benchmark:} when $N=28$, we first exhaustively discovered the optimal sequence. After $14.38$ hours of searching, the global optimal sequence was found to be
\begin{equation}
\bm{s}_\textup{opt}=
\begin{bmatrix}
-1 & +1 & -1 & +1 & -1 & +1 & -1  \\[0.002mm]
+1 & -1 & -1 & +1 & -1 & -1 & +1  \\[0.002mm]
+1 & +1 & +1 & -1 & -1 & -1 & -1  \\[0.002mm]
-1 & -1 & -1 & -1 & -1 & -1 & -1
\end{bmatrix}.\nonumber
\end{equation}
It yields optimal metric $\mathcal{M}^*(\bm{s}_\textup{opt})\approx 30.02$.

\noindent\textbf{AlphaSeq:} The parameter settings for AlphaSeq are listed in Table \ref{TableR1}.
\begin{table}[t]
\setlength{\abovecaptionskip}{3pt}
\centering
\caption{}
\label{TableR1}
\includegraphics[scale=1]{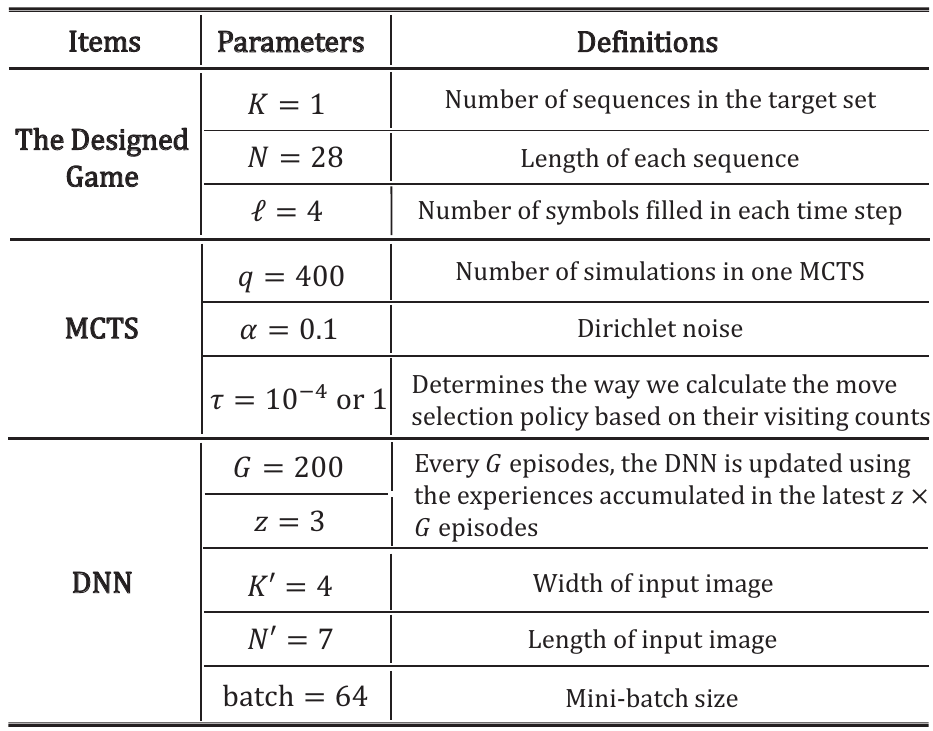}
\end{table}
We aimed to discover one sequence of length $28$ wherein $K=1$ and $N=28$ in the AlphaSeq game. The number of symbols filled in each time step was set to $\ell=4$, and the DNN taked $4\times 7$ images as input. A complete sequence can be obtained after $7$ time steps ($28$ symbols). Then, we calculated the metric function and reward function following equations (18) and (19) in our paper. The search range of $\mathcal{M}(\bm{s})$ was set to $\mathcal{M}_l=0$ and $\mathcal{M}_u=37$.

Every $G = 200$ episodes, we updated the DNN using the experiences accumulated in the latest $600$ episodes. Specifically, the number of experiences is $600\times\lceil NK/\ell \rceil$, and we randomly sampled $\lceil 600\times\lceil NK/\ell \rceil/64\rceil\times 6$ mini-batches ($64$ is the mini-batch size) with replacement to train the DNN by stochastic gradient descent.

\noindent\textbf{Performances:} Over the course of training, we ran AlphaSeq over $5000$ episodes, generating $3.5\times 10^4$ experiences in total. Every $G=200$ episodes when the DNN is updated, we recorded the time consumed by then, and evaluated the searching capability of AlphaSeq by using it to play $50$ noiseless games. The mean metric $E[\mathcal{M}]$ and maximal metric $\max[\mathcal{M}]$ of these $50$ noiseless games were also recorded.
Fig. \ref{FigR1} presented the $E[\mathcal{M}]$ and $\max[\mathcal{M}]$ versus time consumption over the reinforcement learning (RL) process.
\begin{figure}[t]
  \centering
  \includegraphics[width=0.9\columnwidth]{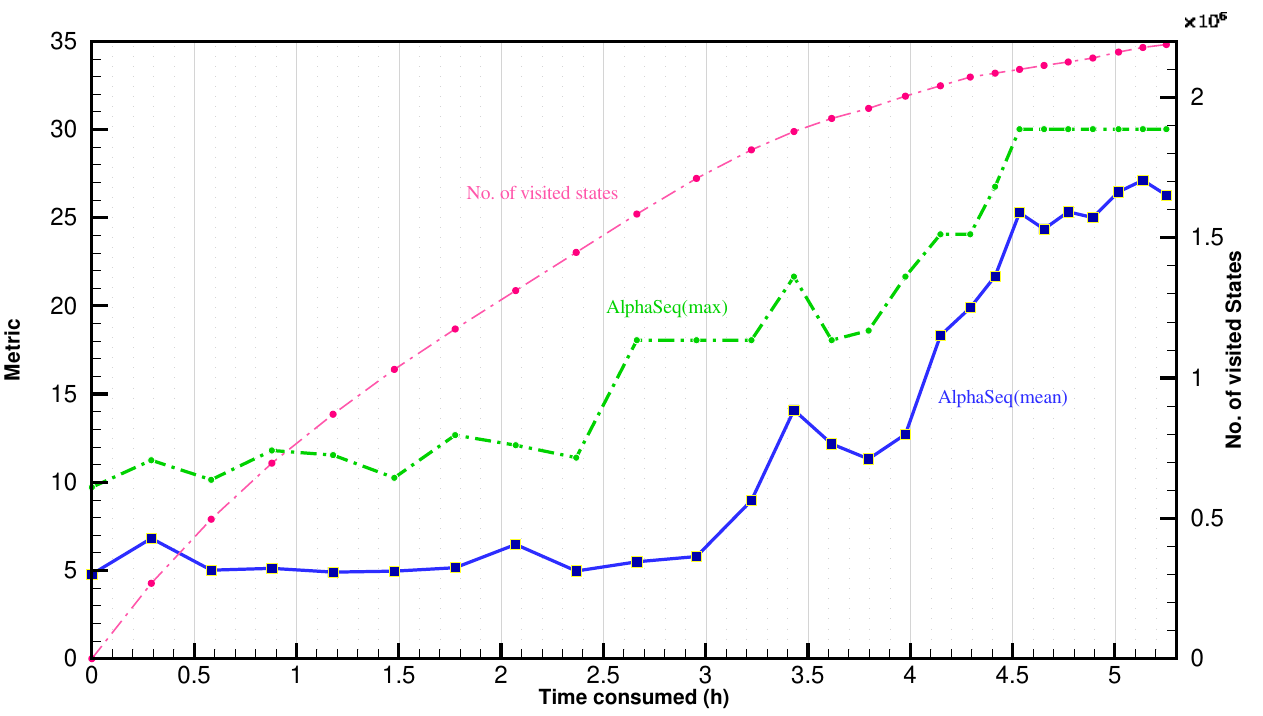}\\
  \caption{The RL process of AlphaSeq to discover a sequence for pulse compression radar. Mean metric $E[\mathcal{M}]$, maximum metric $\max[\mathcal{M}]$ and total number of visited states versus time consumption.}
\label{FigR1}
\end{figure}

For each curve in Fig. \ref{FigR1}, the spacing between two adjacent points is $200$ episodes. The first $3.2$ hours (i.e., from $0$ to $2200$ episodes) are exploration-dominant phase and the rest are the exploitation-dominant phase. As can be seen, in the exploration-dominant phase, every $200$ episodes takes around the same amount of time; while in the exploitation-dominant phase, the time consumed by every $200$ episodes decreases as time goes on. This is because the most time-consuming part in AlphaSeq is MCTS. Throughout the exploration-dominant phase, MCTS actively explores new states in its choice of actions, hence the time consumed by each $200$ episodes are almost the same. On the other hand, in the exploitation-dominant phase, the policy of DNN converges, and AlphaSeq tends to focus more on exploitation. Thus, the number of states visited in every $200$ episodes decreases as time goes on, and the time consumption also decreases.

After about $4.5$ hours, AlphaSeq discovered the optimal sequence with $\mathcal{M}(\bm{s}_\textup{Alpha})\approx 30.02$.
\begin{equation}
\bm{s}_\textup{Alpha}=
\begin{bmatrix}
-1 & +1 & -1 & +1 & -1 & +1 & -1  \\[0.002mm]
+1 & -1 & -1 & +1 & -1 & -1 & +1  \\[0.002mm]
+1 & +1 & +1 & -1 & -1 & -1 & -1  \\[0.002mm]
-1 & -1 & -1 & -1 & -1 & -1 & -1
\end{bmatrix}.\nonumber
\end{equation}

\section{}\label{sec:AppE}
In this paper, we formulated the sequence discover problem as a Markov decision process (MDP), and put forth a deep reinforcement learning (DRL)-based paradigm, AlphaSeq, to solve the MDP.

Note that a major difference between the sequence discovery problem and a typical RL problem is that the dynamics of the environment are fully known to the agent in the sequence discovery problem. For example, when taking action $a_t$ in state $s_i$, where the environment will be steered to is known to the agent in sequence discovery. This is different from typical RL problems like Shannon's mouse~\cite{shannon}, or playing Atari games~\cite{Atari2015}, where the agent has no prior knowledge of the environment.

Since the dynamics of the environment are known to the agent, it is in fact possible for the agent to perform look-ahead simulations when in a state, and use the information collected in the simulations to improve its real action in this state.

Of all the state-of-the-art DRL methods, the reason we follow the DRL framework of AlphaGo is ¡°Monte-Carlo Tree Search (MCTS)¡±. Instructed by current DNN, the MCTS in AlphaGo expands the search tree and finds the most promising action from its cognition of subsequent states. This is an advantage that no other DRL algorithms have.

As requested by the reviewer, we develop another DRL approach, named DQLSeq, to solve the MDP associated with the sequence discovery problem. As the name suggests, DQLSeq follows another state-of-the-art DRL algorithm, deep Q learning (DQL).

\subsection{The Framework of DQLSeq}
\begin{figure}[t]
	\centering
	\includegraphics[width=0.6\columnwidth]{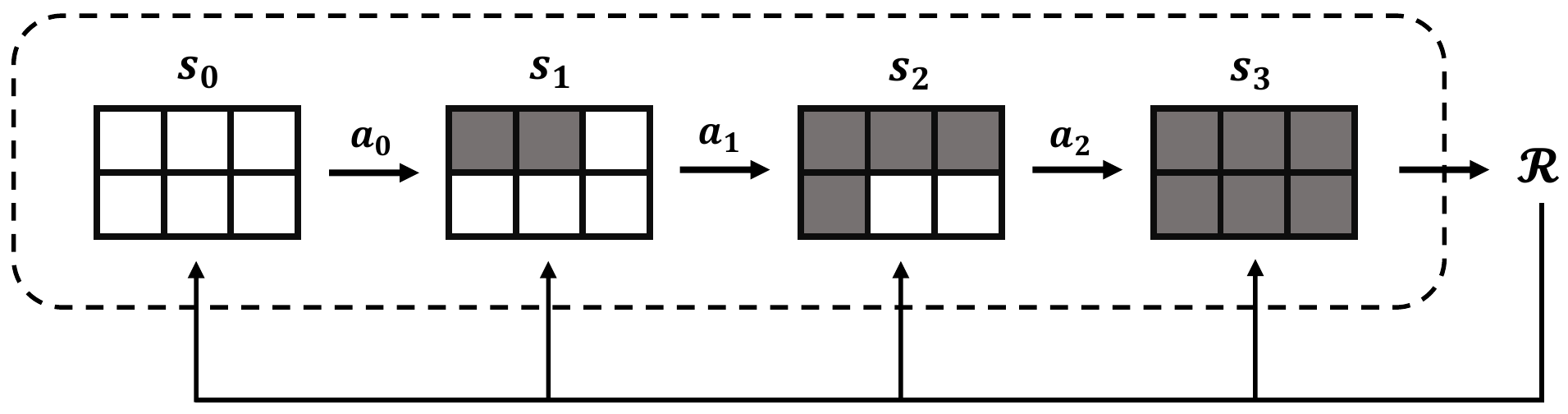}
	\caption{An episode of game, where $K=2$, $N=3$, and $\ell=2$. The $NK$ positions are represented by the colored squares: grey means that the positions are filled while white means that the positions are vacant.}
	\label{fig:1}
\end{figure}

As a refresher, we treat sequence-set discovery as a symbol filling game. One play of the game is one episode, and each episode contains a series of time steps. As illustrated in Fig.~\ref{fig:1}, the agent starts from an all-zero state in each episode, and takes one action per time step: $\ell$ symbols in the sequence set are assigned with the value of $1$ or $-1$, replacing the original $0$ value. An episode ends at a terminal state after $\left \lceil  NK/\ell\right \rceil$ time steps, whereupon a complete set $\mathcal{C}$ is obtained. We then measure the goodness of $\mathcal{C}$ by a reward $\mathcal{R}$ for this episode to the agent. The agent's objective is to learn a policy that makes sequential decisions to maximize the reward, as more and more games are played.

A key component in DQL is the deep Q network (DQN). The DQN is parameterized by $\lambda$, and outputs the Q-value estimations $Q(s,a;\lambda)$ for all state-action pair $(s,a)$ when a state $s$ is provided as its input, i.e., $Q(s,a;\lambda)= \Psi(s;\lambda)$.

In a state $s_i$, the action to be executed is determined by the epsilon-greedy algorithm~\cite{RLbook}. The detailed DQLSeq is explained in Algorithm~\ref{algo:1}.

{\centering
\begin{minipage}{15cm}
\begin{algorithm}[H]
\SetAlCapFnt{\small}
\SetAlCapNameFnt{\small}
\small
\setstretch{0.9}
    \underline{\textbf{Initialization:}}\\
    \hspace{0.2in} Initialize a FIFO of size $M$ for experience replay. \\
    \hspace{0.2in} Initialize a deep Q network $Q(s, a; \lambda)=\Psi(s;\lambda)$ randomly with parameters $\lambda$.\\
    \hspace{0.2in} Set mini-batch size to $K$, evaluation cycle to $B$, episode index $d = 1$. \\
    \While{$1$}{
        \underline{\textbf{Experience collection:}}\\
        time step $i=0$ \\
        \While{$i<\lceil NK/\ell\rceil$}{
        Feed state $s_i$ into the DQN for the Q-value estimations $Q(s_i, a; \lambda)$, $\forall~a$ in state $s_i$. \\
        Select action $a_i$ by the epsilon-greedy algorithm
        \begin{equation}
        a_i=\left\{
                \begin{aligned}
                \arg\max_a Q(s_i, a; \lambda) && \textup{w.p.}~ 1-\varepsilon \\
                \textup{random action} &&  \textup{w.p.}~\varepsilon
                \end{aligned}
                \right.
        \end{equation}
        Execute $a_i$, and observe new environment state $s_{i+1}$. \\
        $i=i+1$
        }
        $d=d+1$. \\
        Compute reward $\mathcal{R}$ for the terminal state of this episode. \\
        Store experiences $\{s_i, a_i, \mathcal{R}\}$ for each intermediate state of this episode into FIFO. \\
        \underline{\textbf{DQN training:}}\\
        \If{ $d > M/2$}{
            Randomly sample a mini-batch of $K$ experiences from the FIFO. \\
            For thie mini-batch $\{s_k,a_k, \mathcal{R}_k\}$, $k\in\{i_1,i_2,\cdots, i_K\}$, update parameters $\lambda$ in the direction that minimizes loss
            \begin{equation}
                \mathcal{L}=\frac{1}{K} \sum_{k\in\{i_1,\dots, i_K\}} \left[\mathcal{R}_k-Q(s_k, a_k; \lambda)\right]^2.
            \end{equation}
        }
        \underline{\textbf{Performance Monitoring:}}\\
        \If{$\text{mod}(d, B) == 0$}{
                Temporarily set $\varepsilon=0$ and run one episode of game. \\
                Compute and record the performance of the discovered sequence to evaluate the current DQN.
            }
    }
    \caption{DQLSeq.}
    \label{algo:1}
\end{algorithm}
\end{minipage}
\par
}

\subsection{Performance}
To compare the performance of AlphaSeq and DQLSeq, we repeat the experiment in Section IV using DQLSeq. That is, we will make use of DQLSeq to discover phase-coded sequences for pulse compression radar systems, and compare the performance of the discovered sequence with that discovered by AlphaSeq.

\begin{figure}[t]
	\centering
	\includegraphics[width=0.7\columnwidth]{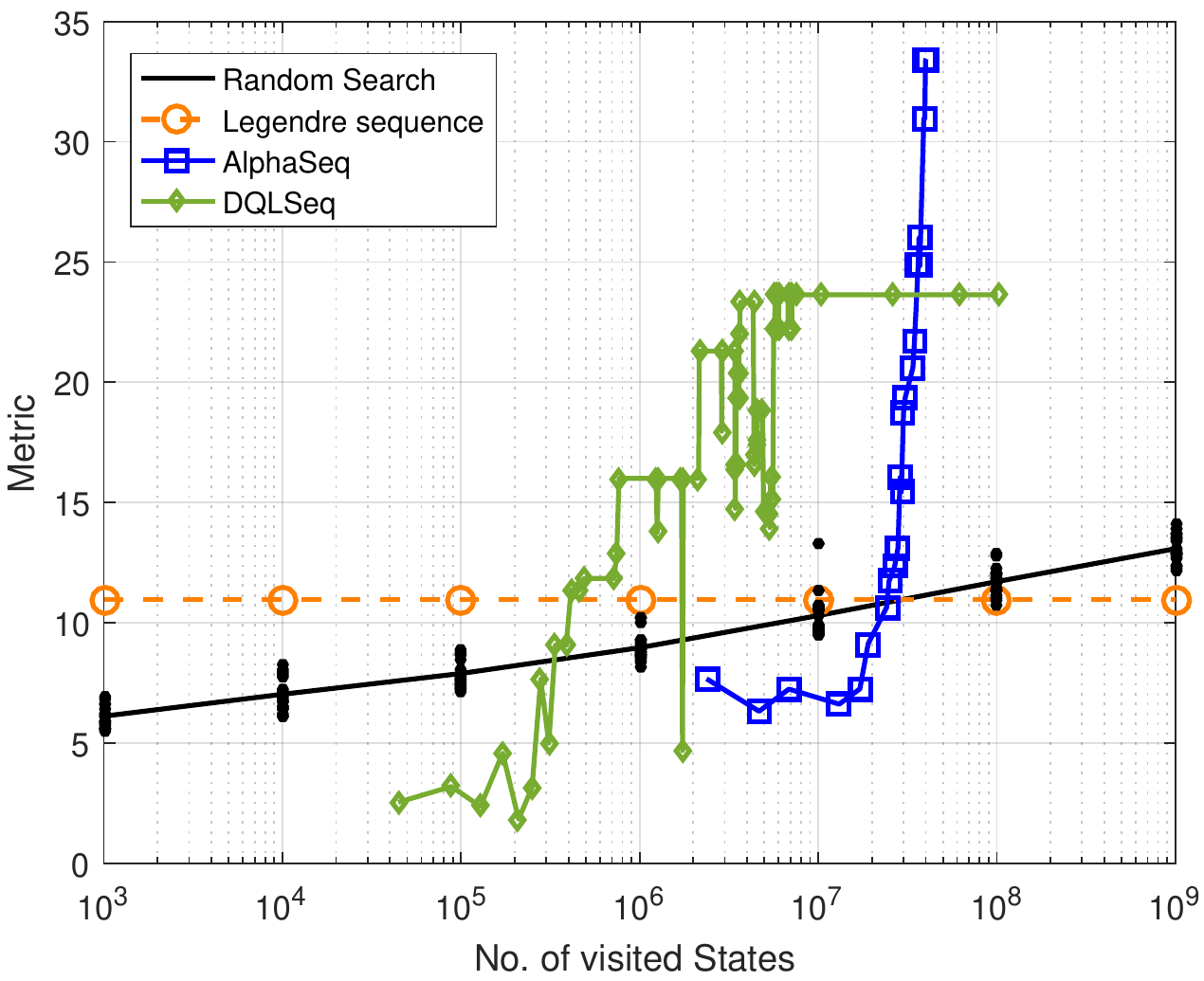}
	\caption{The searching capability comparison of AlphaSeq and DQLSeq.}
	\label{fig:2}
\end{figure}

In the implementation, every $B$ episodes, we will make use of DQLSeq to discover one sequence, and compute the metric $\mathcal{M}(\bm{s}_\textup{DQL})$ by Eq. (18) for the discovered sequence $\bm{s}_\textup{DQL}$. This metric value will be recorded along with the number of distinct visited states by then. We specify that the visited states include both intermediate states and terminal states, as AlphaSeq in Fig. 8.

Fig.~\ref{fig:2} presents the metric of the discovered sequence versus the number of distinct visited states curve for DQLSeq in the learning process. The other curves for random search, Legendre sequence, and AlphaSeq are copied from Fig. 8.

As can be seen, DQLSeq converges faster than AlphaSeq, but is prone to getting stuck in local optimum. The best sequence found by DQLSeq is given by
\begin{equation}
\bm{s}_\textup{DQL}=
\begin{bmatrix}
\begin{smallmatrix}
+1 & -1 & +1 & +1 & -1 & +1 & -1 & +1 & +1 & -1 & +1 & -1 \\[0.1cm]
+1 & +1 & -1 & +1 & -1 & +1 & +1 & -1 & +1 & +1 & +1 & -1 \\[0.1cm]
-1 & -1 & +1 & +1 & -1 & +1 & -1 & -1 & -1 & +1 & -1 & +1 \\[0.1cm]
+1 & +1 & +1 & +1 & -1 & -1 & -1 & -1 & -1 & -1 & -1 & -1 \\[0.1cm]
-1 & -1 & -1 & -1 & -1 & -1 & -1 & -1 & -1 & -1 & -1 &
\end{smallmatrix}
\end{bmatrix},\nonumber
\end{equation}
the metric of which is $\mathcal{M}(\bm{s}_\textup{DQL})=23.64$. On the other hand, the metric of the best sequence found by AlphaSeq is $\mathcal{M}(\bm{s}_\textup{DQL})=33.45$.

\bibliographystyle{IEEEtran}
\bibliography{References}

\end{document}